\documentclass{article}

\PassOptionsToPackage{authoryear, round}{natbib} %numbers,square
\usepackage[main, final]{Styles/neurips_2026}

\usepackage[utf8]{inputenc} % allow utf-8 input
\usepackage[T1]{fontenc}    % use 8-bit T1 fonts
\usepackage{url}            % simple URL typesetting
\usepackage{booktabs}       % professional-quality tables
\usepackage{amsfonts}       % blackboard math symbols
\usepackage{nicefrac}       % compact symbols for 1/2, etc.
\usepackage{microtype}      % microtypography
\usepackage{xcolor}         % colors

\title{Sampling-Based Safe Reinforcement Learning}

\author{%
  Luca Vignola\thanks{Corresponding author: \texttt{lvignola@ethz.ch}} \\
  ETH Zurich\\
\And
Bruce D.~Lee\\
  ETH Zurich\\
  \And
Manish Prajapat 
\\
ETH Zurich \\
\AND
Manuel Wendl 
\\
ETH Zurich \\
\And
Melanie Zeilinger
\\
ETH Zurich \\
\And
Andreas Krause
\\
ETH Zurich \\
\And
Yarden As \\
ETH Zurich
}

\usepackage{amsmath,amsthm}
\usepackage{graphicx}
\usepackage[T1]{fontenc}
\usepackage[utf8]{inputenc}
\usepackage[inline]{enumitem}
\usepackage{dsfont}

\usepackage[english]{babel}
\usepackage{amssymb,amsfonts}
\usepackage{bbm}
\usepackage{nicefrac}
\usepackage{xfrac}
\usepackage{subcaption}
\usepackage{booktabs}
\usepackage{multirow}
\usepackage{tikz}
\usepackage{tikz-3dplot}
\usetikzlibrary{3d,arrows.meta,patterns,decorations.markings,calc,positioning}
\usepackage[colorinlistoftodos,bordercolor=purple,backgroundcolor=pink!20,linecolor=pink,textsize=scriptsize]{todonotes}
\usepackage{wrapfig}
\usepackage{microtype}      % microtypography
\usepackage{parskip}
\usepackage{caption}
\usepackage{needspace}
\usepackage{placeins}

\usepackage{thm-restate}
\usepackage{empheq}
\usepackage{color}
\definecolor{cool-blue}{HTML}{1D6996}
\definecolor{cool-purple}{HTML}{5F4690}
\definecolor{cool-teal}{HTML}{38A6A5}
\definecolor{cool-green}{HTML}{0F8554}
\definecolor{dark-green}{rgb}{0.0,0.5,0.0}

\newcommand{\algname}[1]{{\small\sf#1}}
\usepackage[pagebackref]{hyperref}
\hypersetup{
	colorlinks=true,
	linkcolor=cool-blue,
	filecolor=cool-blue,
    citecolor=cool-green,
	urlcolor=cool-blue
}
\renewcommand*{\backrefalt}[4]{%
    \ifcase #1 \footnotesize{(Not cited.)}%
    \or        \footnotesize{(Cited on page~#2)}%
    \else      \footnotesize{(Cited on pages~#2)}%
    \fi}

\usepackage{algorithm}
\usepackage{algpseudocode}
\usepackage[capitalize]{cleveref}
\usepackage[theorems,skins]{tcolorbox}
\usepackage{accents}
\usepackage{stackengine}

\stackMath          % <- makes stack commands use math mode

\newtheorem{definition}{Definition}
\crefname{definition}{definition}{definitions}
\newtheorem{assumption}{Assumption}
\crefname{assumption}{assumption}{assumptions}

\crefname{theorem}{theorem}{theorems}
\newtheorem{corollary}{Corollary}
\crefname{corollary}{corollary}{corollaries}
\newtheorem{lemma}{Lemma}
\crefname{lemma}{lemma}{lemmas}
\theoremstyle{remark}

\crefname{remark}{remark}{remarks}

\newcommand{\calA}{\ensuremath{\mathcal{A}}}

\newcommand{\calD}{\ensuremath{\mathcal{D}}}

\newcommand{\calF}{\ensuremath{\mathcal{F}}}
\newcommand{\calG}{\ensuremath{\mathcal{G}}}
\newcommand{\calH}{\ensuremath{\mathcal{H}}}

\newcommand{\calX}{\ensuremath{\mathcal{X}}}

\newcommand{\calZ}{\ensuremath{\mathcal{Z}}}

\newcommand{\curly}[1]{\mathopen{}\left\{#1\right\}\mathclose{}}
\newcommand{\brac}[1]{\left[#1\right]}

\usepackage[colorinlistoftodos,bordercolor=purple,backgroundcolor=pink!20,linecolor=pink,textsize=scriptsize]{todonotes}

\newcommand{\norm}[1]{\ensuremath{\left\| #1 \right\|}}

\begin{document}

\maketitle

\begin{abstract}
\looseness=-1
Safe exploration remains a fundamental challenge in reinforcement learning (RL), limiting the deployment of RL agents in the real world. 
We propose \emph{Sampling-Based Safe Reinforcement Learning} (\algname{SBSRL}), a model-based RL algorithm that maintains safety throughout the learning process by enforcing constraints jointly across a \emph{finite} set of dynamics samples. 
This formulation approximates an intractable worst-case optimization over uncertain dynamics and enables practical safety guarantees in continuous domains. 
We further introduce an exploration strategy based on constraining epistemic uncertainty, eliminating the need for explicit exploration bonuses.
Under regularity conditions, we derive high-probability guarantees of safety throughout learning and a finite-time sample complexity bound for recovering a near-optimal policy. 
Empirically, \algname{SBSRL} achieves safe and efficient exploration both in simulation and in real robotic hardware, and readily extends to practical deep-ensemble implementations that scale to high-dimensional continuous control problems.
\end{abstract}

\section{Introduction}

\looseness=-1 Reinforcement learning (RL) is a framework for sequential decision-making that has achieved remarkable success across a wide range of application domains~\citep{minh,silver2017mastering,RLrobotics,RLroboticsTang,DegraveFusion,ouyang2022traininglanguagemodelsfollow}.
Despite this progress, training RL policies \emph{during deployment} remains challenging due to two fundamental barriers: \emph{sample efficiency} and \emph{safety}.
Standard RL algorithms typically require a large number of interactions with the environment, which is often impractical or costly in physical systems. 
Moreover, ensuring safety during training is itself a central challenge: because of uncertainty about the system dynamics, the consequences of any action are uncertain; in this sense, every interaction with the environment is inherently exploratory, and may cause catastrophic failures.
Together, these challenges limit the applicability of RL largely to settings that can be accurately simulated or where high-quality demonstration data is available.

\looseness=-1
A key challenge in \emph{safe exploration}~\citep{amodei2016concrete} is to satisfy constraints despite uncertainty, while simultaneously visiting uncertain regions of the state-action space to gather informative data and improve performance. 
Model-based RL provides a natural framework for addressing this tension by explicitly modeling uncertainty in the learned environment dynamics. Existing approaches \citep{as2025actsafe,wendl2026safe} typically use such uncertainty estimates to pessimistically tighten constraints to guarantee safety throughout learning, and to add uncertainty-driven bonuses to encourage exploration. However, these methods often 
\begin{enumerate*}[label=\textbf{(\roman*)}]
 \item introduce excessive conservatism through their constraint tightening, and 
 \item require careful tuning of exploration bonuses to ensure both convergence and effective exploration.
\end{enumerate*}

\looseness=-1 To address these challenges, we propose \emph{Sampling-Based Safe Reinforcement Learning} (\algname{SBSRL}), a model-based algorithm for safe and sample-efficient exploration.
\algname{SBSRL} learns an uncertainty-aware dynamics model, leveraging its \emph{epistemic} uncertainty to drive exploration while also maintaining safety at all times.
Sampling-based methods provide flexible uncertainty estimates, scale to expressive model classes \citep{chua2018deep} and admit natural parallelization.
\algname{SBSRL} leverages these properties to obtain a tractable approximation to the robust optimization problem at the core of safe exploration: finding a policy that satisfies safety constraints under \emph{worst-case} dynamics among a set of plausible models of the environment (see \Cref{pb:intractable}). Rather than solving this worst-case problem explicitly, \algname{SBSRL} enforces safety on each sample of the dynamics, as we illustrate in \Cref{fig:hardware_flagship}.
In addition, rather than adding an exploration bonus to the objective, \algname{SBSRL} encodes exploration as a constraint that requires the agent to collect sufficient information along its trajectories. As a result, exploration is triggered only when a greedy policy is not informative enough, yielding a simpler and more robust strategy for safe exploration.
\definecolor{obstacle-red}{HTML}{FF0533}
\definecolor{obstacle}{HTML}{38A6A5}
\definecolor{truef}{HTML}{5F4690}
\definecolor{goal}{HTML}{1D6996} %{73AF48}{EDAD08}
\begin{wrapfigure}[29]{r}{0.5\textwidth}
    \centering
    \vspace{0.5cm} 
    \begin{tikzpicture}
        \node[inner sep=0] (img) {\includegraphics[width=\linewidth]{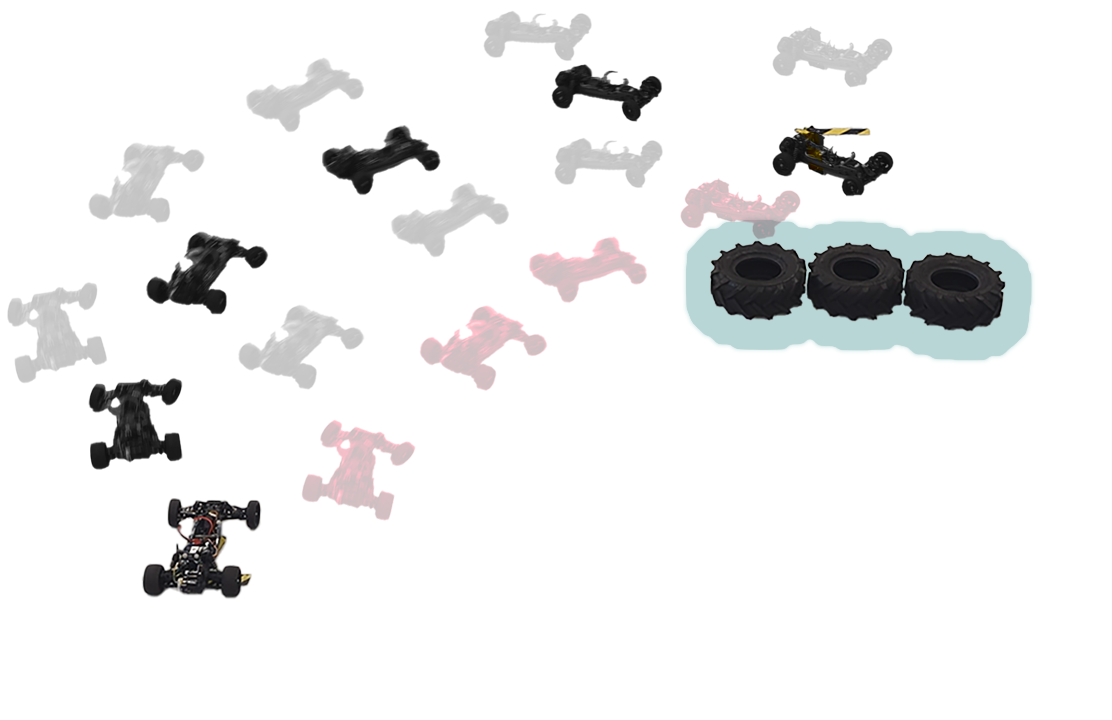}};

        \begin{scope}[shift={(img.north west)},
                      x={(img.north east)},
                      y={(img.south west)}]

            \node[
                fill=white, rounded corners=2pt, inner sep=2pt,
                font=\small, align=left, text=goal
            ] (tightlbl) at (0.93,0.12)
            {Goal};
            \draw[-{Stealth[length=2mm]}, goal, line width=1pt]
                (tightlbl.west) to[out=180,in=45] (0.79,0.18);
            
            \node[
                fill=white, rounded corners=2pt, inner sep=2pt,
                font=\small, align=left, text=obstacle
            ] (tightlbl) at (0.60,0.70)
            {tightened constraint};

            \draw[-{Stealth[length=2mm]}, obstacle, line width=1pt]
                (tightlbl.north) to[out=90,in=230] (0.75,0.50);
% labels for trajectories already visible in the image
            \node[
                fill=white, fill opacity=0.9, text opacity=1,
                rounded corners=2pt, inner sep=1.5pt,
                font=\scriptsize
            ] at (0.13,0.13) {$f_n^1$};

            \node[
                fill=white, fill opacity=0.9, text opacity=1,
                rounded corners=2pt, inner sep=1.5pt,
                font=\scriptsize
            ] at (0.33,0.39) {$f_n^2$};
            
            \node[
                fill=white, fill opacity=0.9, text opacity=1,
                rounded corners=2pt, inner sep=1.5pt,
                font=\scriptsize, text=obstacle-red
            ] at (0.46,0.56) {$f_n^M$};

            \node[
                fill=white, fill opacity=0.9, text opacity=1,
                rounded corners=2pt, inner sep=1.5pt,
                text=truef
            ] at (0.26,0.28) {$f^{\star}$};

        \end{scope}
    \end{tikzpicture}
    \caption{\looseness-1 We illustrate a safe exploration task in which a car with unknown dynamics must learn to reach a goal while avoiding obstacles (tires). At each episode $n$, \algname{SBSRL} maintains a set of $M$ dynamics samples $f_n^1,\dots,f_n^M$ and enforces \textcolor{obstacle}{tightened} safety constraints jointly across their corresponding model-generated rollouts, denoted by transparent cars. 
For instance, the depicted policy would lead to constraint violations under the plausible model \textcolor{obstacle-red}{$f_n^M$}. \algname{SBSRL} updates its policy by simulating it in dynamics $f_n^1,\dots,f_n^M$, thereby ensuring safety under the true system dynamics \textcolor{truef}{$f^\star$}. As more episodes are collected and uncertainty shrinks, the sampling-based approximation becomes progressively less conservative, enabling \algname{SBSRL} to discover higher-reward trajectories and ultimately reach the goal.}
\label{fig:hardware_flagship}
\end{wrapfigure}
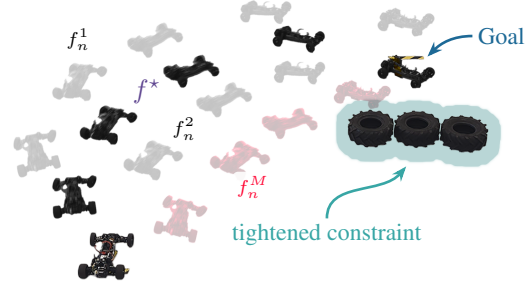

\paragraph{Our contribution.}
\begin{itemize}[leftmargin=0.5cm, parsep=0.3em]
    \item We propose \algname{SBSRL}, a novel model-based RL algorithm for safe exploration in continuous state-action spaces. \algname{SBSRL} enforces safety using a finite number of dynamics samples, and promotes exploration via a constraint on the epistemic uncertainty.
    \item We show that when the dynamics are modeled using a Gaussian process, \algname{SBSRL} guarantees safety with high probability. In addition, we provide a sample-complexity bound showing that \algname{SBSRL} obtains near-optimal policies in a finite number of episodes.
    \item \looseness-1 We empirically validate \algname{SBSRL} both in simulation and on hardware. With Gaussian process dynamics models, consistent with our theory,
    \algname{SBSRL} achieves safe and efficient exploration. We also introduce a scalable neural ensemble variant that enforces constraints jointly across multiple dynamics and evaluate it in SafetyGym and RWRL \citep{Ray2019,challengesrealworld}. Finally, we empirically show that these principles carry over to real-world hardware, enabling safe online learning in practice.
\end{itemize}

\section{Related Work}\label{section:related_works}

\paragraph{Safety in CMDPs.}
\looseness=-1 Safety in RL has been modeled in various ways \citep{Garca2015ACS, brunke2022safe, gu2024review}. 
Among these, constrained Markov decision processes (CMDPs) \citep{altman} provide a natural framework, as they enjoy many classical results from planning in MDPs, and can be used to formulate different notions of safety \citep{Wagener2021SafeRL, curi2022safe, NEURIPS2023_5d4cd12e}.
Many works address learning and planning in CMDPs, both in discrete~\citep{vaswani2022nearoptimal,Ding2022ConvergenceAS,noRegretCMDP} and continuous~\citep{cpo,as2022constrained,sootla2022saute,huang2024safedreamer} state-action spaces. However, these algorithms ensure safety only at convergence, and may violate constraints during learning.
This problem of safely learning, i.e., safe exploration, has been tackled in tabular CMDPs ~\citep{moldovan2012safe,turchetta2016safe,Wachi_Sui_Yue_Ono_2018,wachi2020safereinforcementlearning,ding2021provably,bura2022dope}, where strong guarantees are available, but these approaches do not scale well to continuous problems. Extensions to continuous state-action spaces have been explored in the model predictive control (MPC) literature~\citep{Koller2018LearningBasedMP,hewing2019cautious,wabersich2021predictive,manish,prajapat2024towards}. 
However, these techniques rely on online planning, requiring the solution of an optimization problem at each step and thus incurring significant computational cost.

\paragraph{Sampling-based algorithms.}
\looseness=-1 Sampling-based methods are popular in robotics for their efficiency and parallelizability \citep{tobin2017domain,reachabilityanalysis22,reachabilityrobotics}. Examples in constrained settings include scenario optimization \citep{scenario2006,2018ecc..conf..228U}, which provides
violation-probability guarantees under a specified distribution (cf. Appendix~\ref{appendix:discussions}); conformal prediction \citep{Lindemann2025FormalVA}, which typically yields post-hoc guarantees; sampling-based policy-gradient
methods that sample actions rather than
dynamics \citep{sampling-based-policy-gradients}; and dynamics sampling within MPC \citep{prajapat2024towards,prajapat2025finite}. Closely related to our work, \citet{prajapat2025finite} sample functions from GP dynamics models and enforce constraints jointly across all sampled dynamics within an MPC framework. We instead apply this principle in episodic finite-horizon
CMDPs 
and establish expected-cost safety
throughout learning together with finite-time near-optimality. Existing CMDP methods \citep{as2022constrained,as2025actsafe} typically tackle epistemic uncertainty via worst-case optimization over plausible models; however, this optimization is generally intractable and approximated via heuristics, sacrificing the original theoretical guarantees. 
We close this gap by jointly solving the CMDP across all sampled dynamics and establishing theoretical guarantees directly for the resulting algorithm.

\paragraph{Epistemic exploration and optimality.}
\looseness=-1 In online learning, purely greedy approaches under a nominal model may explore inadequately, motivating optimism-based methods \citep{NIPS2006_c1b70d96,Srinivas_2012,HURCL}. 
In constrained settings, however, optimism alone is not sufficient, since the feasible set may fail to expand enough to include the optimal solution \citep{sui2015safe,hubotter2024transductive}.
In the CMDP setting, \citet{as2025actsafe} propose \algname{ActSafe}, an explore-then-commit strategy \citep{garivier2016explore}, with simple regret optimality guarantees. However, such pure exploration approaches prioritize information gathering and may sacrifice reward performance in the early stages of learning. More recently, \citet{wendl2026safe} develop an algorithm based on intrinsic rewards~\citep{sukhija2025sombrl} that achieves sublinear cumulative regret. While theoretically grounded, intrinsic reward methods are often sensitive in practice to the choice of hyperparameters, potentially leading to over-exploration or premature exploitation. In contrast, our approach encourages exploration through a constraint on the epistemic uncertainty, allowing the policy to remain greedy whenever the constraint is satisfied. Akin to our work \citet{prajapat2025safe} adopts an exploration constraint formulation, however in an MPC framework, whereas we establish simple regret guarantees in the CMDP setting.

\section{Problem Setting}\label{section:problem_setting}
\subsection{Constrained Markov Decision Processes}
\looseness=-1
We consider a discrete-time finite-horizon CMDP defined by the tuple $\langle\mathcal X, \mathcal A, T, p, r, c, d, \rho_0\rangle$. The sets $\mathcal{X} \subseteq \mathbb R^{d_x}$ and $\mathcal{A} \subseteq \mathbb R^{d_a}$ denote the state and action spaces, respectively. The episode horizon is $T$. The reward and cost functions are assumed to be bounded and given by $r: \mathbb{R}^{d_x} \times \mathbb{R}^{d_a} \to [0, R_{\max}]$ and $c: \mathbb{R}^{d_x} \times \mathbb{R}^{d_a} \to [0, C_{\max}]$, respectively. 
The initial state distribution is denoted as $\rho_0$ and the transition probability $p(\cdot \mid x,a)$ is induced by the stochastic dynamics
\begin{align}
    \label{eq: dyn}
    x_{t+1} = f^\star(x_t, a_t) + w_t, \quad t = 0, \dots, T-1, 
\end{align}
where $f^\star$ is an unknown function and $\{w_t\}_{t = 0}^{T-1}$ is a stochastic process taking values in $\mathcal{W} \subseteq \mathbb{R}^{d_x}$, representing process noise.
The agent selects actions according to a policy $\pi:\calX\to\calA$ from a class $\Pi$. While finite-horizon MDPs require time-dependent policies~\citep{puterman2014markov}, we omit the dependence on $t$ for notational brevity.

Given a policy $\pi$ and a transition function $f$, the noisy dynamics \eqref{eq: dyn} induces a trajectory distribution from which we define 
the expected reward return $J_r(\pi, f)$ and cost return $J_c(\pi, f)$, where the return of an arbitrary function $g: \calX \times \calA \to \mathbb{R}$ is given by 
\begin{align}
    \label{eq: return}
    J_g(\pi, f) \coloneq \mathbb{E}_{\tau_\pi^f}\brac{\sum_{t=0}^{T-1} g(x_t, a_t)}.
\end{align}
Here, the subscript ${\tau_\pi^f}$ in the expectation denotes a trajectory rolled out according to the policy $a_t = \pi(x_t)$ and the dynamics $x_{t+1} = f(x_t, a_t) + w_t$. The expectation is taken over the initial state $x_0\sim \rho_0$ and the process noise sequence $\{w_t\}_{t=0}^{T-1}$. The goal of the learning agent is to find a policy that maximizes the expected reward return subject to the expected cost remaining below a threshold $d$, both evaluated under the true, unknown dynamics $f^\star$: 
\begin{equation}
\label{eq: synthesis}
\begin{aligned}
\pi^\star \in \arg\max_{\pi \in \Pi} J_r(\pi, f^\star) \quad 
\text{s.t.} \quad J_c(\pi, f^\star) \le d.
\end{aligned}
\end{equation}
\subsection{Objective}
Since $f^\star$ is unknown, the agent interacts with the system over multiple episodes to acquire information. During episode $n$, the agent executes policy $\pi_n$ and collects a trajectory of length $T$ consisting of transitions $\curly{(x_t^n, a_t^n, x_{t+1}^n)}_{t=0}^{T-1}$.  
Importantly, beyond the final objective in \eqref{eq: synthesis}, we require safety \emph{during learning}: the policy $\pi_n$ selected \emph{at each episode} must satisfy the constraint despite not knowing the true dynamics, i.e.,
\begin{align}
    J_c(\pi_n, f^\star) \leq d  \quad\forall n\in\{0,1,\dots\}. 
\end{align}

Our goal is to design an algorithm that collects data to progressively reduce uncertainty about the system until \eqref{eq: synthesis} can be solved up to arbitrary tolerances. The algorithm is designed to provably
\begin{enumerate*}[label=\textbf{(\roman*)}]
\item satisfy the constraint on the expected cost return at every episode,
\item terminate after a \emph{finite} number of episodes, and 
\item achieve a near-optimal return upon termination.
\end{enumerate*}
To this end, we make the following assumptions, which are standard in the safe MBRL literature (see, e.g., Assumptions 4.1-4.5 in \citet{as2025actsafe} and Assumptions 5.1-5.2 in \citet{sukhija2025sombrl}).

\subsection{Assumptions}

\begin{assumption}[Gaussian noise] \label{assump:process_noise}
    \looseness=-1The process noise is i.i.d.\ Gaussian, with $
    w_t \;\overset{\text{i.i.d.}}{\sim}\; \mathcal{N}(0,\,\sigma_w^2 I)$.
\end{assumption}
This assumption can be relaxed to sub-Gaussian noise as in \citet{HURCL} by assuming instead Lipschitz continuous costs and rewards, true dynamics, and policies. The resulting sample complexity would suffer an exponential-in-horizon amplification governed by the Lipschitz constants. 
Instead, we assume Gaussian noise to invoke Kakade et al.'s ``Difference of Gaussians'' simulation lemma~\citep{KakadeInfo} and obtain tighter sample complexity bounds.
\looseness -1
\paragraph{Prior knowledge.}
In safe exploration, absences of prior knowledge makes safety generally ill-posed, since no policy can be certified as safe without structural assumptions on the unknown dynamics. This requirement is standard in the safe learning literature, where it is typically enforced by assuming access to an initially calibrated model and a corresponding set of safe policies \citep{as2025actsafe,wendl2026safe}. We therefore require a prior model class that captures admissible deviations from a nominal dynamics estimate.
\begin{assumption}[RKHS with bounded norm]\label{assump:rkhs} Let $\calZ\coloneqq\calX\times\calA$.
We assume access to a known prior mean function $\mu : \calZ \to \mathbb{R}^{d_x}$ and kernel $k: \calZ \times \calZ \to \mathbb{R}$ such that, for each component $f^{\star}_j$ of the true dynamics $f^{\star} = (f^{\star}_1,\dots,f^{\star}_{d_x})$, we have that $f_j^\star-\mu_j$ is an element of the RKHS $\calH_k$ induced by the kernel $k$ and has a known bound $B>0$ on the RKHS norm, i.e. $f_j^\star-\mu_j \in  \mathcal{H}_k$
and $\|f_j^\star - \mu_j\|_k \leq B,\; \forall \,j\in[d_x]$.
Moreover, the kernel satisfies
$k(z,z) \leq \sigma_{\max}$ $\forall\, z \in \calZ.$
\end{assumption}
\looseness-1 Assumption~\ref{assump:rkhs} formalizes that the prior model $\mu$ is accurate up to a bounded error measured in RKHS norm. We further require that the prior model class induced by $(\mu, k, B)$ admits at least one policy that is uniformly safe over a class of initial admissible dynamics, defined as $\calF_0^B\coloneqq\{f\mid|f_j(z)-\mu_j(z)|\leq B\sqrt{k(z,z)} \;\forall \,j\in[d_x],z\in\calZ\}$. Note that, by definition of the RKHS norm, every dynamics function satisfying Assumption \ref{assump:rkhs} must also lie in $\calF_0^B$.

\begin{assumption}[Feasible safe initialization]
\label{assump:prior_knowledge}
There exists a constant $\Delta > 0$ such that the set
    $\Pi_\text{prior}\coloneqq\{\pi\mid \;J_c(\pi,f)\leq d-\Delta \; \forall f \in\calF_0^B\}$
is non-empty.
\end{assumption}
\looseness=-1Together, Assumptions~\ref{assump:rkhs} and \ref{assump:prior_knowledge} encode that the prior knowledge is both sufficiently expressive to contain the true dynamics and sufficiently informative to certify at least one safe policy.
In practice, the prior mean $\mu$, and its corresponding safe initialization, can be obtained in different ways, such as from an offline dataset $\calD_0$ or a simulator, with $B$ capturing the resulting model mismatch or sim-to-real gap.
This makes the framework agnostic to the source of prior information, as long as a meaningful bound on the discrepancy is available.

To establish continuity of the learned dynamics model and ensure finite sample complexity guarantees, we require two additional regularity assumptions on the prior mean and kernel.
\begin{assumption} \label{assumption:continuity_sigma}
    The prior mean $\mu$ and kernel $k$ are Lipschitz continuous on $\calZ$ with constants $L_\mu$, $L_k$.
\end{assumption}
To characterize the complexity of the function class $\calH_k$, we introduce the notion of maximum information gain for a kernel $k$ with $N$ observations from the set $\calZ$ as 
$\gamma_N(k)
\coloneq \max_{\calG \subset \calZ,\, |\calG| \le N} 
\frac{1}{2} \log \det \bigl(I + \sigma_w^{-2} K_{\calG}\bigr),$
where $K_{\calG} \coloneq [\, k(z, z') \,]_{z,z' \in \calG}$. This quantity measures the maximum reduction in uncertainty achievable after $N$ observations.
To ensure that the exploration complexity remains finite, we further restrict attention to kernels whose information gain grows sufficiently slowly with $N$, as formalized in the following assumption.
\begin{assumption}\label{assump:information_gain}
\looseness=-1The maximum information gain $\gamma_{N}(k)$ of the kernel $k$ over the set $\calZ$ satisfies 
$\gamma_{N}^3(k)=o(N)$.
\end{assumption}
\Cref{assumption:continuity_sigma,assump:information_gain} are mild and are satisfied by common kernel choices such as linear and squared exponential, as well as Matérn kernels with sufficiently large smoothness parameters, when composed with bounded and Lipschitz continuous feature maps (cf. Appendix \ref{appendix:discussions}).

\section{Background}
\label{section:background}
\paragraph{Gaussian process dynamics models.} \looseness-1 For our analysis, we model the unknown dynamics using Gaussian processes (GPs) built around a known prior mean function $\mu : \mathcal{Z} \to \mathbb{R}^{d_x}$ and kernel $k$. At episode $n$, the GP is updated using all data collected in earlier episodes.
For each episode $\ell \in\{0,\dots,n-1\}$ and stage $t \in \{0,\dots,T-1\}$, let 
$z_t^\ell := (x_t^\ell,a_t^\ell) \in \calZ := \calX \times \calA$ and 
$y_t^\ell := x_{t+1}^\ell \in \calX$. The dataset available at episode $n$ is
$\calD_{0:n-1} = \bigcup_{\ell=0}^{n-1} \calD_\ell, 
\calD_\ell = \{(z_t^\ell, y_t^\ell)\}_{t=0}^{T-1}.$
Each state coordinate can be modeled independently with a GP. For $j \in [d_x]$, define the output vector $\mathbf{y}_{n-1,j} := [\, y_{t,j}^\ell \,]_{(\ell,t)}$, the collection of training inputs
$\mathbf Z_{0:n-1}:=[\,z_t^\ell\,]_{(\ell,t)}$, the corresponding
prior-mean vector
$\mu_j(\mathbf Z_{0:n-1}):=[\,\mu_j(z_t^\ell)\,]_{(\ell,t)}$, the kernel vector 
$\mathbf k_{n-1}(z) := [\, k(z_t^\ell, z) \,]_{(\ell,t)}$, and the Gram matrix 
$\mathbf K_{n-1} := [\, k(z_t^\ell, z_{t'}^{\ell'}) \,]_{(\ell,t),(\ell',t')}$, 
where $(\ell,t)$ ranges over $\{0,\dots,n-1\} \times \{0,\dots,T-1\}$.
Given the prior mean function $\mu_j$ and the kernel $k$, the posterior mean and variance for component $j$ are obtained via standard GP regression:
\begin{equation}
\label{eq: gp}
\begin{aligned}
    \mu_{n,j}(z)
&=\mu_j(z) + \mathbf k_{n-1}(z)^\top (\mathbf K_{n-1} + \sigma_w^2 I)^{-1} \big(\mathbf y_{n-1,j}-\mu_j(\mathbf Z_{0:n-1})\big) \\
\sigma_{n,j}^2(z) 
&= k(z,z) - \mathbf k_{n-1}(z)^\top (\mathbf K_{n-1} + \sigma_w^2 I)^{-1} \mathbf k_{n-1}(z). 
\end{aligned}
\end{equation}
The vector-valued mean and standard deviation are denoted by $\mu_n(z) := (\mu_{n,1}(z),\dots,\mu_{n,d_x}(z))$ and 
$\sigma_n(z) := (\sigma_{n,1}(z),\dots,\sigma_{n,d_x}(z))$. We further define the scalar uncertainty measure $s_n: \calZ \to \mathbb{R}$ with $s_n(z) := \norm{\sigma_n(z)}$.
The following Lemma then establishes that a GP model is \emph{well-calibrated}, in the sense that the true dynamics lie within the confidence intervals induced by the posterior mean and variance with high probability.

\begin{lemma}[Vector-valued analog of Theorem 2 of \citet{Chowdhury}] \label{lemma:confidence_gp}
Let \Cref{assump:process_noise,assump:rkhs} hold. For any $\delta \in (0,1]$ and $n \geq 0$, define $
  \beta_{n}(\delta) \coloneqq 
  B +  \sqrt{2(\gamma_{nT}(k) + 1
  +  \ln\!\left(\tfrac{d_x}{\delta})\right)}.$
Then with probability at least $1-\delta$ it holds that for all $n=0,1, 2, ...$, for all $z\in \calZ$ and for all $j\in[d_x]$ that
    $|\mu_{n,j}(z) - f^\star_j(z)| \leq \beta_n(\delta) \sigma_{n,j}(z).$
\end{lemma}
These confidence regions are typically used in the literature to conservatively enforce constraints \citep{as2025actsafe}. In our approach, we approximate these regions via sampling.
\paragraph{Approximation via GP sampling.} 
\looseness-1 We  define a vector-valued \emph{sample} $\tilde f=(\tilde f_1, \dots, \tilde f_{d_x})\sim GP(\mu,k)$ from the prior GP by concatenating $\tilde f_j \sim GP(\mu_j, k)$ for $j \in [d_x]$.
The key idea underlying our approach is that, by drawing sufficiently many samples, we can ensure with high probability that at least one is uniformly close to the true dynamics. 
 Drawing continuous function realizations $\tilde f\sim GP(\mu,k)$ is computationally intractable due to their infinite dimensional nature. However, since we only need to evaluate these functions at the discrete, finite state-action pairs visited by the agent, several methods can efficiently generate GP samples in practice (cf. Appendix~\ref{appendix:discussions}). We now formalize a sufficient number of samples to achieve at least one that is $\zeta$-close. To do so, we first introduce the \emph{small-ball probability}.

\begin{definition}[Small-ball probability; \citet{van2011information}]\label{def:smallball} Let $\tilde f_j \sim GP(0,k)$. The \emph{small-ball probability} is defined by $\Pr\brac{\norm{\tilde f_j}_{\infty} < \zeta} =: \exp(-\phi(\zeta)) $ for any $\zeta >0 $ where the probability is given by the \emph{small-ball exponent} $\phi: \mathbb R \to \mathbb R$.
\end{definition}
The small-ball exponent can be bounded using properties of the kernel $k$. 
See \citet{prajapat2025finite} for bounds for common kernel classes.
\begin{restatable}[Theorem 1 of \citet{prajapat2025finite}]{lemma}{samplesthm}
\label{lem: sampling}
Let \Cref{assump:rkhs} hold.
    Consider any $\delta \in (0,1]$ and $
    \zeta > 0$. Select $M$ satisfying
    \begin{align*}
        M \geq\frac{\log(\delta)}{\log(1-\exp(-d_x (\frac{1}{2}B^2 + \phi(\zeta))))}.
    \end{align*}
    Let $\tilde f^1, \dots, \tilde f^M \sim GP(\mu, k)$ be function samples from the prior GP.
    It holds with probability at least $1-\delta$ that for some $m\in[M]$, $ \norm{\tilde f_j^m- f_j^\star}_{\infty} \leq \zeta \;\;\forall j\in[d_x]$.
\end{restatable}
\looseness-1 See Appendix~\ref{appendix:technical} for the proof of Lemma \ref{lem: sampling}.
By characterizing the number of samples required to achieve a $\zeta$-close approximation to an arbitrary function in the RKHS, the above lemma provides a conservative \emph{worst-case} bound on the number of GP samples required so that enforcing constraints on \emph{sampled} dynamics implies that they also hold on the \emph{true unknown dynamics}. 
While this bound can be large, it is required to formally guarantee safety with high-probability for a very general class of problems. 
In practice, we observe that significantly fewer samples are sufficient to obtain safe policies.
Next, we describe how this bound is used in our algorithm design.
 
\section{Algorithm}\label{section:algorithm}
\paragraph{Enforcing safety via sampling.}\looseness=-1 
Given a well-calibrated model (\Cref{lemma:confidence_gp}), satisfying the safety constraint in each episode $n$ amounts to solving a worst-case optimization given by
\begin{equation} \label{pb:intractable}
\begin{aligned}
\max_{f \in\calF_{n}}J_c(\pi, f) \le d,
\end{aligned}
\end{equation}
where $\calF_n\coloneqq \{f\colon \, |\mu_{n,j}(z) - f_j(z)| \leq \beta_n(\delta) \sigma_{n,j}(z) \;\forall j\in[d_x]\}$.
Prior work in safe MBRL addresses this in practice via heuristics~\citep{as2025actsafe,as2022constrained} or pessimistic cost penalties~\citep{wendl2026safe}.
We instead approximate this optimization over the class $\calF_n$ with a finite set of $M$ sampled models $\Tilde{f}^1, \dots, \Tilde{f}^{M}$ drawn i.i.d. from the GP prior.
For sufficiently large $M$ (see \Cref{lem: sampling}), at least one sample is $\zeta$-close to $f^\star$ with high probability, allowing us to ensure safety on the true system by enforcing constraints over the sampled models. 

\paragraph{Truncating unlikely samples.}
\looseness=-1Since these samples are drawn from the prior and not yet informed by data, we refine them online by truncation~\citep{prajapat2024towards}. At every episode $n$, each sample is truncated to lie within the intersection of all previous GP confidence sets,
\begin{align}
    \label{eq: truncated}
    f_{n,j}^m(z) = \mathrm{clip}(
        f_{n-1,j}^m(z), \mu_{n,j}(z) - \beta_{n}(\delta) \sigma_{n,j}(z), \mu_{n,j}(z) + \beta_{n}(\delta) \sigma_{n,j}(z) ),
\end{align}
for all $z \in \mathcal{Z}, j \in [d_x]$. For $n=0$, we overload the notation and apply Equation \ref{eq: truncated} to $f_{-1,j}^m\coloneqq\Tilde{f}_{j}^m$, $\mu_0(z)\coloneqq\mu(z)$, $\sigma_{0,j}(z)\coloneqq \sqrt{k(z,z)}$, and $\beta_0(\delta)\coloneqq B$, yielding $f_0^m\in\calF_0^B \,\forall m\in[M]$.
\begin{wrapfigure}[16]{r}{0.5\columnwidth}
    \vspace{-0.1cm}
    \centering
    \includegraphics[width=0.95\linewidth]{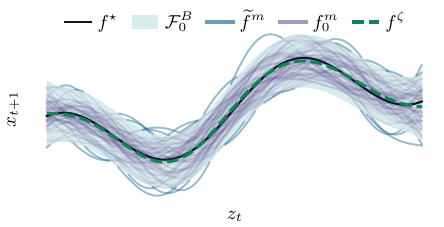}
    \caption{At the initial episode $n=0$, we draw $M$ samples $\Tilde{f}^m$ from the GP prior. These are then clipped using the prior bound $\mathcal{F}_0^B$, yielding the truncated samples $f_0^m$. Finally, $f^\zeta$ denotes the $\zeta$-close sample of Lemma \ref{lem: sampling}.}
    \label{fig:gp_samples}
    \vspace{-1em}
\end{wrapfigure}
For well-calibrated models (cf. \Cref{lemma:confidence_gp}), truncation can only reduce deviation from $f^\star$, thereby preserving the $\zeta$-close sample. This idea is illustrated in \Cref{fig:gp_samples}.

\paragraph{Constraint tightening.}
To account for the gap between this sample and $f^\star$, we then tighten the constraint by a $\zeta$-dependent constant and enforce it jointly across all samples:
\begin{align*}
    J_c(\pi,  f_{n}^m) \leq d - \Delta_\zeta \quad \forall m=1,\dots, M,
\end{align*}
which guarantees safety under the true dynamics with high probability. Intuitively, the tightening factor $\Delta_\zeta$, defined in Algorithm \ref{alg:mbrl1}, bounds the discrepancy between returns under $f^\star$ and a $\zeta$-close model. 
This yields a natural tradeoff: increasing $\zeta$ reduces the number of samples $M$ required by Lemma~\ref{lem: sampling} but increases conservatism, while smaller $\zeta$ improves tightness at higher computational cost. We note that truncation is only one approach for refining GP samples with data. For other approaches, including posterior sampling, we refer the reader to \Cref{appendix:discussions}.

\paragraph{Enforcing exploration as a constraint.}
\looseness=-1 Recalling the definition of $s_n(z)\coloneqq \norm{\sigma_n(z)}$, we denote the corresponding expected cumulative uncertainty by $J_{s_n}$, as defined in \Cref{eq: return}. Prior safe MBRL methods often encourage exploration via intrinsic reward bonuses, where the agent maximizes
$J_r(\pi,\mu_{n}) + \lambda_n^{\textrm{explore}} J_{s_{n}}(\pi,\mu_{n})$
\citep{wendl2026safe}. In practice, these approaches require careful tuning of $\lambda_n^{\textrm{explore}}$, and may either over-explore irrelevant regions or fail to explore sufficiently under model misspecification. In contrast, we decouple reward maximization and exploration by enforcing a constraint that the epistemic uncertainty accumulated along trajectories under our nominal dynamics exceeds a threshold. Concretely, at episode $n$, we require the policy to satisfy $J_{s_n}(\pi_n, \mu_n) \geq d_\sigma^n$, 
where $d_\sigma^n$ is a prescribed exploration threshold defined in Theorem \ref{thm:merged}.
Importantly, this constraint becomes inactive when the greedy policy is already sufficiently informative, in which case no additional exploration is enforced. 
Moreover, infeasibility of the constraint provides a natural stopping criterion, certifying that remaining policies induce uniformly low epistemic uncertainty and linking the exploration threshold directly to the final suboptimality guarantee.

\paragraph{Sampling-Based Safe Reinforcement Learning (\algname{SBSRL}).}
\looseness-1 We now propose a model-based reinforcement learning algorithm that combines sampling-based constraint satisfaction with the exploration constraint. In particular, as shown in \Cref{alg:mbrl1}, at each episode $n$, the agent selects the policy $\pi_{n}$ as 
\begin{equation} \label{pb:1}
\begin{aligned}
\pi_{n} = \arg\max_{\pi \in \Pi}\; J_r(\pi, \mu_n)\;\;\text{s.t.}\;\;
 J_c(\pi, f_n^m) \le d - \Delta_\zeta\; \forall m \in[M]
\;\text{and}\;\, J_{s_n}(\pi, \mu_n) \ge d_\sigma^n.
\end{aligned}
\end{equation}

{\fontsize{7}{8}\selectfont
\begin{algorithm}[H] 
\caption{Sampling-Based Safe Reinforcement Learning (SBSRL)}
\label{alg:mbrl1}
\begin{algorithmic}[1]
  \Require Confidence parameter $\delta$, closeness threshold $\zeta\in(0,\frac{\sigma_w\Delta}{\sqrt{d_x}T^2C_{\max}})$, exploration threshold sequence $\curly{d_{\sigma}^n}_{n \geq 0}$, prior mean $\mu$ and kernel $k$
  \State Dataset $\mathcal{D}_0\gets\varnothing$, episode $n \gets 0$ 
  \State Set constraint tightening $\Delta_\zeta \gets \zeta \sqrt{d_x}\frac{T^2 C_{\max}}{\sigma_w}$
  \State Set number of dynamics samples $M$ to satisfy the condition of \Cref{lem: sampling} under $\mu, k, \delta, \zeta$ 
  \State Draw $M$ dynamics samples $\tilde f^1, \dots, \tilde f^M \sim \text{GP}(\mu, k)$
  \While{Problem \eqref{pb:1} is feasible}
    \State Construct truncated dynamics samples $f_n^1, \dots, f_n^M$ according to  \eqref{eq: truncated}
    \State Select $\pi_{n}$ by solving Problem \eqref{pb:1} 
    \State $\mathcal{D}_{0:n} \leftarrow \mathcal{D}_{0:n-1} \cup \text{rollout } \pi_{n}$ 
    \State Update model $\text{GP}(\mu_{n+1},\sigma_{n+1}) \leftarrow \mathcal{D}_{0:n}$
    \State $n \gets n+1$
  \EndWhile
  \State Select $\hat \pi$ by solving Problem \eqref{pb:1} after removing the constraint on $J_{s_n}$ (greedy policy). 
\end{algorithmic}
\end{algorithm}
}
\looseness=-1
In the following, we present our theoretical guarantees for \algname{SBSRL}. We provide further details regarding our practical implementation of \algname{SBSRL} in  \Cref{appendix:experiments}.

\section{Results}\label{section:results}
\looseness=-1 We now prove that with high probability, \algname{SBSRL} satisfies the constraint at all episodes, terminates after a finite number of steps, and is nearly optimal upon termination. We defer proofs to the supplement. 

\looseness=-1Our first result shows that solving Problem~(\ref{pb:1}) guarantees safety of the policy $\pi_n$ at all episodes $n$ with high probability. For this, define the set policies which satisfy the safety constraint of Problem~(\ref{pb:1}) as
\begin{equation*}
    \Pi^n_{\text{safe}}\coloneq\{\pi \in \Pi \quad | \quad  J_c(\pi,f_n^m) \le d - \Delta_\zeta
   \quad \forall m = 1,\dots,M\}.
\end{equation*}

\begin{restatable}[Safety]{theorem}{safetythm}\label{thm:safety} 
Suppose Assumptions \ref{assump:process_noise}-\ref{assump:prior_knowledge} hold. Consider any $\delta \in (0, 1/2)$, $\zeta \in(0,\frac{\sigma_w\Delta}{\sqrt{d_x}T^2 C_{\max}})$, and sequence of positive numbers $\curly{d_{\sigma}^{n}}_{n\geq 0}$. Consider running \Cref{alg:mbrl1} given these values.  
  Then, it holds with probability at least  \( 1 - 2\delta \) that for every episode $n$, any policy $\pi \in \Pi_{\textrm{safe}}^n$ satisfies
      $J_c(\pi, f^\star) \leq d.$
  In particular, this holds for the policy $\pi_n$ selected during episode $n$ and the policy $\hat \pi$ selected upon algorithm termination. 
\end{restatable}
\looseness-1 Moreover, under the success event of \Cref{lemma:confidence_gp}, \Cref{assump:prior_knowledge} ensures that the sampled safe set $\Pi^n_{\mathrm{safe}}$ is non-empty for every episode $n\geq0$ (cf. Appendix \ref{appendix:discussions}). Consequently, infeasibility can only occur due to the exploration constraint. This is studied in the following theorem, which bounds the number of episodes played by the algorithm before termination, leading to near-optimality guarantees.
We next introduce the set of $\xi$-tightened safe policies under the true dynamics,
\begin{equation*}
    \Pi_{\xi}^{\star}\coloneq\{\pi \in \Pi \quad | \quad  J_c(\pi, f^\star) \le d - \xi\}.
\end{equation*}
 In general, $\Pi_{\xi}^{\star}$ may be disconnected, with some regions that cannot be reached from the initial safe set $\Pi_\text{prior}$ via safe policy updates. We therefore restrict our attention to a \emph{connected} component $\Pi_{\xi}^{\star,c}\subseteq\Pi_{\xi}^{\star}$, defined as follows. 
\begin{definition} \label{assump:reachability} We define $\Pi_{\xi}^{\star,c} \subseteq \Pi_{\xi}^{\star}$ as the subset of $\Pi_{\xi}^{\star}$ satisfying the following path-connectedness condition with respect to $\Pi_\text{prior}$: for every $\pi\in\Pi_{\xi}^{\star,c}$ there exists a continuous path
    $\rho:[0,1]\to\Pi_{\xi}^{\star,c}$ such that $\rho(0)=\pi$ and  $\rho(1)\in\Pi_\text{prior}$.
    Continuity is with respect to the norm $\|\cdot\|_{\infty}$, defined as,
    $\|\pi - \pi'\|_{\infty} \coloneqq \sup_{x \in \mathcal{X}} \|\pi(x) - \pi'(x)\|,$
where $\|\cdot\|$ denotes the Euclidean norm in the action space. 
\end{definition}
Similar notions of reachability are standard in the safe online learning literature~\citep{turchetta2016safe, Berkenkamp2017SafeMR, prajapat2025safe}. 

\begin{restatable}[Optimality]{theorem}{mergedthm} \label{thm:merged}
Suppose Assumptions \ref{assump:process_noise}-\ref{assump:information_gain} hold. Fix some $\delta\in(0,1/2)$, $\zeta \in(0,\frac{\sigma_w\Delta}{\sqrt{d_x}T^2C_{\max}})$, $\varepsilon > 0$ and $\xi>\varepsilon+ 
\zeta\sqrt{d_x} 
\frac{T^2 C_{\max}}{\sigma_w}$.
Let $d_\sigma^n=\frac{\varepsilon\sigma_w}{2G_{\text{max}} T\beta_n(\delta)}$ with $G_{\text{max}}\coloneqq\max\{C_\text{max},R_\text{max}\}$, and consider running \Cref{alg:mbrl1}.
There exists a problem-dependent constant $C = C(d_x, \sigma_{\max}, G_{\max}, T, \delta)$ such that if $\bar n$ is the smallest integer satisfying 
$\bar n \;\ge\; C \, \frac{\gamma_{\bar n T}(k)^3}{\varepsilon^2}$, then with probability at least $1-2\delta$ the following hold:
\begin{enumerate}[noitemsep,nolistsep,leftmargin=0.5cm, parsep=0.3em]
    \item There exists $n^\star \leq \bar n$ such that
    $J_{s_{n^\star}}(\pi, \mu_{n^\star}) < d_\sigma^{n^\star} \,\, \forall \pi \in\Pi^{{n^\star}}_{\text{safe}},$
    i.e. Algorithm~\ref{alg:mbrl1} terminates at iteration ${n}^\star$.
    \item At termination, the policy $\hat\pi$ returned by \Cref{alg:mbrl1} satisfies
    \begin{equation} \label{bound:optimality}
    \max_{\pi \in \Pi_{\xi}^{\star,c}} J_r(\pi, f^\star) - J_r(\hat \pi, f^\star)
    \leq \varepsilon.   
    \end{equation} 
\end{enumerate}
\end{restatable}

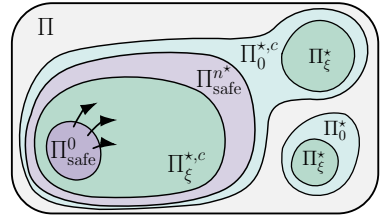
\begin{wrapfigure}[14]{r}{0.5\columnwidth}
    \vspace{-1.5em}
    \centering
    \resizebox{0.43\columnwidth}{!}{\begingroup
\definecolor{safeInner}{HTML}{5F4690}
\definecolor{safeMid}{HTML}{38A6A5}
\definecolor{safeOuter}{HTML}{5F4690}
\definecolor{policyFill}{HTML}{f2f2f2}
% Keep the same code for now, but allow independent colors for Pi_0* and Pi_xi* families.
\definecolor{starZeroColor}{HTML}{38A6A5}
\definecolor{starXiColor}{HTML}{0F8554}

% Opacity controls (edit these to tune transparency in one place)
\newcommand{\alphaStarZeroC}{0.2}  % Pi_0^{\star,c}
\newcommand{\alphaStarZero}{0.2}   % Pi_0^{\star}
\newcommand{\alphaStarXi}{0.3}     % Pi_{\xi}^{\star} and Pi_{\xi}^{\star,c}
\newcommand{\alphaSafeOuter}{0.25}  % Pi_{\mathrm{safe}}^{n^{\star}}
\newcommand{\alphaSafeInner}{0.4}  % Pi_{\mathrm{safe}}^{0}

% Opaque fills that visually match alpha-over-white, so color does not depend on lower layers.
\pgfmathsetmacro{\pctStarZeroC}{100*\alphaStarZeroC}
\pgfmathsetmacro{\pctStarZero}{100*\alphaStarZero}
\pgfmathsetmacro{\pctStarXi}{100*\alphaStarXi}
\pgfmathsetmacro{\pctSafeOuter}{100*\alphaSafeOuter}
\pgfmathsetmacro{\pctSafeInner}{100*\alphaSafeInner}
\colorlet{starZeroCFill}{starZeroColor!\pctStarZeroC!white}
\colorlet{starZeroFill}{starZeroColor!\pctStarZero!white}
\colorlet{starXiFill}{starXiColor!\pctStarXi!white}
\colorlet{safeOuterFill}{safeOuter!\pctSafeOuter!white}
\colorlet{safeInnerFill}{safeInner!\pctSafeInner!white}

\begin{tikzpicture}[
    x=1cm,
    y=1cm,
    line cap=round,
    line join=round,
    >={Latex[length=4.2mm, width=2.8mm]}
]

\pgfdeclarelayer{piZeroLayer}
\pgfdeclarelayer{piXiLayer}
\pgfdeclarelayer{safeOuterLayer}
\pgfdeclarelayer{safeMidLayer}
\pgfdeclarelayer{safeInnerLayer}
\pgfdeclarelayer{policyLayer}
\pgfsetlayers{policyLayer,piZeroLayer,piXiLayer,safeOuterLayer,safeMidLayer,safeInnerLayer,main}

% Universal policy set
\begin{pgfonlayer}{policyLayer}
\path[fill=policyFill, rounded corners=28pt]
    (0.26,0.18) rectangle (9.15,5.29);
\end{pgfonlayer}
\draw[rounded corners=28pt, line width=1.3pt]
    (0.26,0.18) rectangle (9.15,5.29);
\node[anchor=south east, font=\LARGE] at (1.42,4.40) {$\Pi$};

\begin{scope}[xshift=-0.26cm, yshift=-0.02cm]
% New surrounding yellow set with a larger top-right expansion
\begin{pgfonlayer}{piZeroLayer}
\filldraw[fill=starZeroCFill, draw=black, line width=1.0pt]
    (1.02,0.79)
    .. controls (0.36,1.88) and (0.80,3.80) .. (1.82,4.1)
    .. controls (2.20,4.26) and (4.28,4.52) .. (5.82,4.44)
    .. controls (6.34,4.48) and (6.66,4.46) .. (6.92,4.74)
    .. controls (7.36,5.18) and (8.06,5.30) .. (8.62,5.04)
    .. controls (9.06,4.84) and (9.26,4.30) .. (9.10,3.82)
    .. controls (8.90,3.22) and (8.34,2.90) .. (7.76,2.98)
    .. controls (7.20,2.96) and (6.82,2.94) .. (6.74,2.42)
    .. controls (6.64,0.96) and (6.36,0.68) .. (5.08,0.42)
    .. controls (4.20,0.34) and (3.42,0.30) .. (2.64,0.3)
    .. controls (1.92,0.36) and (1.38,0.38) .. cycle;
\end{pgfonlayer}
\node[font=\LARGE] at (6.52,4.04) {$\Pi_{0}^{\star,c}$};

% Small yellow blob inside the top-right expansion
\begin{scope}[shift={(7.90,4.02)}, scale=1.65]  % <-- aumenta qui (es. 1.2, 1.3, ...)
\begin{pgfonlayer}{piXiLayer}
\filldraw[fill=starXiFill, draw=black, line width=0.9pt]
    (-0.50,0.04)
    .. controls (-0.54,-0.22) and (-0.38,-0.42) .. (-0.10,-0.50)
    .. controls (0.20,-0.58) and (0.44,-0.42) .. (0.54,-0.14)
    .. controls (0.62,0.10) and (0.52,0.36) .. (0.26,0.48)
    .. controls (-0.04,0.62) and (-0.36,0.44) .. cycle;
\end{pgfonlayer}
\end{scope}
\node[font=\Large] at (8.00,3.92) {$\Pi_{\xi}^{\star}$};

% Separate top-left yellow blob plus concentric outer yellow blob
\begin{scope}[xshift=+5.96cm, yshift=-3.92cm, scale=1.15]
\begin{pgfonlayer}{piZeroLayer}
\filldraw[fill=starZeroFill, draw=black, line width=0.9pt]
    (0.98,4.96)
    .. controls (0.90,4.38) and (1.12,4.08) .. (1.54,3.98)
    .. controls (1.98,3.90) and (2.34,4.12) .. (2.49,4.54)
    .. controls (2.70,5.10) and (2.60,5.50) .. (2.30,5.68)
    .. controls (1.80,5.80) and (1.20,5.50) .. cycle;
\end{pgfonlayer}
\node[font=\Large] at (2.09,5.34) {$\Pi_0^{\star}$};

\begin{pgfonlayer}{piXiLayer}
\filldraw[fill=starXiFill, draw=black, line width=0.9pt]
    (1.18,4.70)
    .. controls (1.12,4.46) and (1.28,4.26) .. (1.56,4.18)
    .. controls (1.84,4.12) and (2.08,4.28) .. (2.14,4.56)
    .. controls (2.20,4.80) and (2.08,5.02) .. (1.82,5.12)
    .. controls (1.52,5.24) and (1.24,5.04) .. cycle;
\end{pgfonlayer}
\node[font=\Large] at (1.64,4.66) {$\Pi_{\xi}^{\star}$};
\end{scope}

\begin{scope}[xshift=-0.08cm, yshift=-0.02cm, scale=0.82]

% Largest safe set after expansion (very close to the middle set)
\begin{pgfonlayer}{safeOuterLayer}
\filldraw[fill=safeOuterFill, draw=black, line width=1.05pt]
    (1.30,3.20)
    .. controls (1.08,2.70) and (1.10,2.10) .. (1.22,1.56)
    .. controls (1.34,0.96) and (2.06,0.60) .. (3.34,0.50)
    .. controls (4.86,0.44) and (6.20,0.58) .. (7.02,0.98)
    .. controls (7.70,1.34) and (7.96,2.04) .. (7.92,2.90)
    .. controls (7.88,3.99) and (7.54,4.74) .. (6.58,4.96)
    .. controls (5.46,5.12) and (3.92,5.04) .. (2.76,4.70)
    .. controls (2.00,4.48) and (1.50,4.02) .. cycle;
\end{pgfonlayer}

% Intermediate certified set
\begin{pgfonlayer}{safeMidLayer}
\filldraw[fill=starXiFill, draw=black, line width=1.0pt]
    (1.72,3.42)
    .. controls (1.44,3.02) and (1.34,2.50) .. (1.44,1.96)
    .. controls (1.54,1.26) and (2.08,0.88) .. (3.10,0.76)
    .. controls (4.36,0.70) and (5.46,0.78) .. (6.16,1.00)
    .. controls (6.76,1.20) and (7.02,1.70) .. (7.04,2.36)
    .. controls (7.06,2.98) and (6.82,3.52) .. (6.32,3.86)
    .. controls (5.58,4.26) and (4.48,4.38) .. (3.52,4.28)
    .. controls (2.74,4.20) and (2.10,3.92) .. cycle;
\end{pgfonlayer}

% Initial safe set
\begin{pgfonlayer}{safeInnerLayer}
\filldraw[fill=safeInnerFill, draw=black, line width=0.95pt]
    (1.92,2.76)
    .. controls (1.74,2.52) and (1.70,2.16) .. (1.84,1.84)
    .. controls (1.98,1.50) and (2.24,1.30) .. (2.62,1.26)
    .. controls (2.98,1.26) and (3.26,1.46) .. (3.34,1.80)
    .. controls (3.42,2.18) and (3.30,2.58) .. (2.98,2.82)
    .. controls (2.62,3.08) and (2.18,3.04) .. cycle;
\end{pgfonlayer}

% Expansion direction cue
\draw[->, line width=1.20pt]
    (2.56,2.84) .. controls (2.66,3.06) and (2.84,3.34) .. (3.28,3.56);
\draw[->, line width=1.20pt]
    (2.98,2.58) .. controls (3.22,2.90) and (3.48,3.02) .. (3.86,3.04);
\draw[->, line width=1.20pt]
    (3.14,2.20) .. controls (3.38,2.30) and (3.58,2.34) .. (3.86,2.30);

% Labels
\node[font=\LARGE] at (2.55,2.12) {$\Pi_{\mathrm{safe}}^{0}$};
\node[font=\LARGE] at (5.90,1.62) {$\Pi_{\xi}^{\star,c}$};
\node[font=\LARGE] at (6.86,4.23) {$\Pi_{\mathrm{safe}}^{n^{\star}}$};

\end{scope}
\end{scope}

\end{tikzpicture}
\endgroup}
    \caption{At termination, $\Pi^{{n^\star}}_{\text{safe}} \supseteq \Pi_{\xi}^{\star,c}$. The set $\Pi_{\xi}^{\star}$ may be disconnected, and only its reachable component $\Pi_{\xi}^{\star,c}$ can be identified via safe updates. The tightening $\xi$ influences connectivity: smaller values can merge otherwise disconnected regions, as illustrated for $\xi=0$. 
    }
    \label{fig:safe_set_expansion}
    \vspace{-1em}
\end{wrapfigure}
\looseness-1 The theorem guarantees near-optimality over the reachable component $\Pi_{\xi}^{\star,c}$ in finite time. In the definition of $\xi$, $\varepsilon$ accounts for the fact that the dynamics can only be learned up to a non-zero error in finite time, while $\zeta$ captures the sampling approximation error and can be reduced by increasing the number of samples.
At termination, the infeasibility of the exploration constraint implies uniformly low epistemic uncertainty over $\Pi^{n^\star}_{\text{safe}}$, allowing us to control the discrepancy between true, mean, and sampled dynamics and ensure that reward and cost are $\varepsilon$-accurate for all safe policies.
In particular, this implies $\Pi^{{n^\star}}_{\text{safe}} \supseteq \Pi_{\xi}^{\star,c}$, as illustrated in Figure~\ref{fig:safe_set_expansion}.

\section{Experiments} \label{section:experiments}
\looseness -1 We empirically evaluate \algname{SBSRL} in two complementary settings:
\begin{enumerate*}[label=\textbf{(\roman*)}]
\item a GP-based instantiation evaluated in simulated environments; and
\item a real-world safe RL setting, where \algname{SBSRL} is implemented using deep ensembles.
\end{enumerate*}
Together, these experiments support our theoretical findings and demonstrate that similar principles can scale in practice to high-dimensional scenarios. 
Further extensive empirical evaluation, including comparisons of our deep variant against state-of-the-art safe exploration algorithms on SafetyGym~\citep{Ray2019} and RWRL~\citep{challengesrealworld}, is deferred to the appendix, together with additional implementation details.
\paragraph{Safety and optimality with GPs.}
\begin{figure}
\vspace{-0.4cm}
    \centering
    \begin{minipage}{0.49\textwidth}
        \centering
\includegraphics{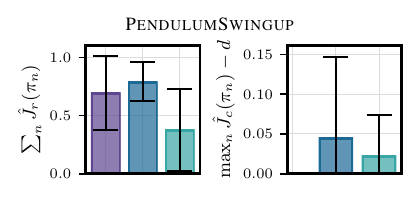}
    \vspace{-0.5cm}
    \end{minipage}
    \hfill
    \begin{minipage}{0.49\textwidth}
        \centering
    \includegraphics{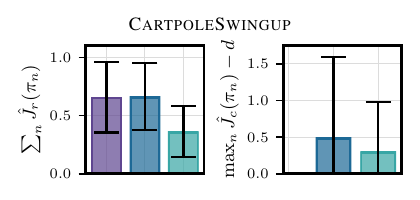}
    \vspace{-0.5cm}
    \end{minipage}
\includegraphics[trim={0 5pt 0 7pt}, clip]{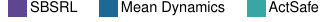}
    \caption{Evaluation of safety on \textsc{PendulumSwingUp} and \textsc{CartpoleSwingUp}. We report mean and standard deviation of the normalized cumulative returns and maximum cost violations over five seeds. \algname{SBSRL} maintains safety throughout training while achieving near-optimal performance.}
    \vspace{-0.5cm}
    \label{fig:experiments_gp}
\end{figure}
\looseness-1 We first conduct an experiment where we use GPs to model the dynamics. We evaluate \algname{SBSRL} on the \textsc{PendulumSwingUp} and \textsc{CartpoleSwingUp} tasks from the Deepmind Control Suite \citep{tassa2018deepmind}, following the setup of \citet{as2025actsafe}. As baselines, we consider \algname{ActSafe} \citep{as2025actsafe} and a baseline where the constraint is enforced only on the GP posterior mean, rather than on each sampled dynamics model as prescribed by Problem~(\ref{pb:1}). Results in \Cref{fig:experiments_gp} show that \algname{SBSRL} achieves near-optimal performance while satisfying the safety constraint at every iteration $n$. In contrast, enforcing constraints on the mean dynamics leads to safety violations in both tasks, highlighting the importance of accounting for model uncertainty at the planning stage, and validating dynamics sampling as a way to ensure safety during learning.

\paragraph{Exploration.}
\looseness=-1 We next validate the exploration strategy of \algname{SBSRL}. Specifically, we consider the \textsc{PendulumSwingUp} task and compare against two variants of \algname{ActSafe} that perform pure exploration for the first $n=3$ and $n=5$ episodes, respectively, before switching to greedy exploitation. 
As shown in Figure \ref{fig:gp_pendulum_exploration}, \algname{SBSRL} shows superior exploration efficiency in this task, converging more rapidly to an optimal policy. This stems from the exploration constraint, which promotes exploration while still allowing reward maximization whenever this is sufficiently informative.
\begin{wrapfigure}[20]{r}{0.42\linewidth}
    \centering
    \vspace{-0.5em}\includegraphics{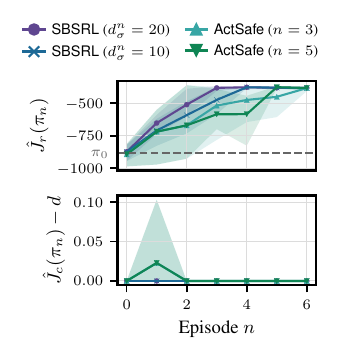}
    \vspace{-2.0em}
    \caption{\looseness=-1Evaluation of exploration via an exploration constraint in the \textsc{PendulumSwingUp} task. We report mean and $95$ percentile interval over five seeds.}
    \label{fig:gp_pendulum_exploration}
\end{wrapfigure}
We further illustrate how the choice of $d_\sigma^0$ can influence the behavior of \algname{SBSRL}.
From our theory, smaller values of $d_\sigma^0$ guarantee closer-to-optimal performance at termination but reduce sample efficiency, which helps explain why increasing its value may actually improve convergence in practice. 
It is also interesting to notice that such a constraint makes little difference in the first training episode, where the greedy policy is already sufficiently informative, consistently with our theoretical insights.

\paragraph{Safe offline-to-online in the real-world.}
\looseness-1 Finally, we demonstrate how \algname{SBSRL} can be successfully deployed on \emph{robotic hardware}, as shown in \Cref{fig:hardware_trajectory} and detailed in Appendix \ref{appendix:experiments}. We instantiate \algname{SBSRL} using deep ensembles (see Appendix \ref{appendix:experiments}) and test it on a highly-dynamic remote-controlled race car operating at $60$ Hz. The combination of high-frequency control and delays in actuation and motion capture makes this high-dimensional environment particularly challenging. This challenge is partially mitigated by using a sliding window of the five past state measurements, resulting in an observation of 45 dimensions.
Having no prior knowledge would make this task ill-posed (recall  \Cref{assump:prior_knowledge}). Therefore, to calibrate our model, we initialize \algname{SBSRL} offline using a dataset of real-world trajectories. We then deploy the algorithm online and compare it with a baseline that only enforces the constraint on the mean critic rather than across the ensemble.
The results in \Cref{fig:hardware_experiment} show that \algname{SBSRL} improves upon the reward of  the prior policy and maintains safety at all times, whereas using only the mean estimate incurs higher costs during training. 
\begin{figure}
    \centering
    \includegraphics[width=\textwidth]{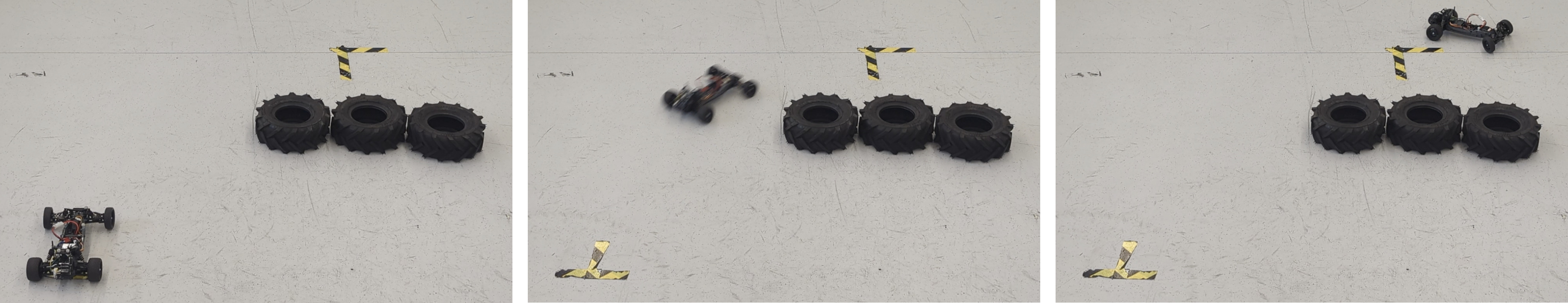}
    \caption{Three frames from a hardware rollout. The car successfully drives from its initial position to the goal behind the obstacle (tyres) while avoiding a collision.}
    \label{fig:hardware_trajectory}
\end{figure}

\begin{figure}
    \centering
    
\includegraphics{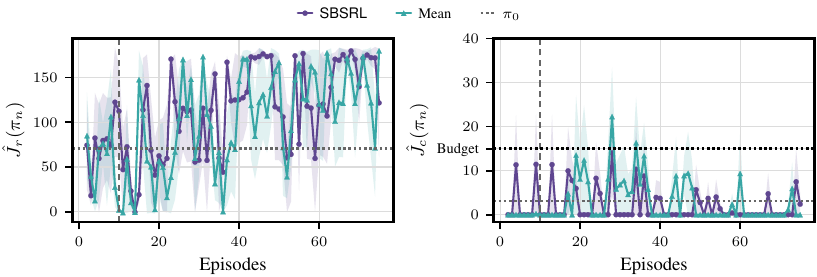}
    \caption{Safe offline-to-online on real-world hardware. We report mean and standard error over three seeds. The vertical line marks the transition from warm-up (offline prior $\pi_0$) to online learning. The cost represents the kinetic energy lost in collisions with the obstacle.}
    \label{fig:hardware_experiment}
    \vspace{-0.5cm}
\end{figure}

\section{Conclusion}\label{section:conclusion}
\looseness=-1 In this work, we propose \algname{SBSRL}, a novel model-based reinforcement learning algorithm for safe exploration. \algname{SBSRL} ensures safety throughout learning by enforcing constraints over a \emph{finite} set of sampled dynamics models. When the dynamics are represented by Gaussian processes, this yields high-probability safety guarantees. Furthermore, we provide a simple regret bound that leverages a novel strategy for encouraging exploration in CMDPs, namely a constraint on the epistemic uncertainty which encourages exploration only when the greedy policy is insufficiently informative.
We extensively validate \algname{SBSRL} both in simulation and on real-world hardware, demonstrating that its deep RL variant can be scaled to high-dimensional continuous control tasks.
Finally, while our work provides a theoretically sound yet practical algorithm for safe exploration in CMDPs, several challenges remain. First, our work builds on the finite-horizon episodic setting, effectively requiring manual resets that are difficult to implement in practice. Extending this work to the non-episodic setting is an exciting direction for future work. In addition, this work builds on the classical expected accumulated cost formulation of constraints~\citep{altman}. Extending this work to state-wise safety formulations~\citep{thomas2021safe}, while retaining high-probability guarantees for high-dimensional nonlinear dynamics and general policies (e.g., beyond MPC), is an important direction for future theoretical and empirical work.

\section*{Acknowledgments}
We would like to thank Armin Lederer and Yunke Ao for insightful discussions during the development of this project.
M.P. was supported by a ETH AI Center Doctoral Fellowship, and B.L. was supported by an ETH AI Center Postdoctoral Fellowship. Y.A. received funding from a grant of the Hasler foundation~(grant no. 21039) and the ETH AI Center. This work was also supported in part by NCCR Automation, a National Centre of Competence in Research, funded by the Swiss National Science Foundation (grant number 51NF40\_225155).

\clearpage

\bibliographystyle{unsrtnat}
\bibliography{references}

%%%%%%%%%%%%%%%%%%%%%%%%%%%%%%%%%%%%%%%%%%%%%%%%%%%%%%%%%%%%

\appendix

\section{Discussion}\label{appendix:discussions}
We collect here several additional discussions and clarifications regarding the assumptions, design choices, and positioning of our work.

\subsection{Positioning of our work}
In classical scenario optimization \citep{scenario2006,2018ecc..conf..228U}, scenarios $\delta^1,\ldots,\delta^M$ are sampled from a distribution $P$ and used to enforce constraints. With high confidence over the sampled scenarios, the resulting decision has a small violation probability under that distribution. 
However, this approach does not inherently certify constraint satisfaction for an arbitrary, fixed underlying dynamics $f^\star$ that is not assumed to be drawn from $P$.
In contrast, \algname{SBSRL} treats $f^\star$ as fixed and assumes that $f^\star-\mu$ belongs to a bounded RKHS ball. Leveraging the result of \citet{prajapat2025finite}, we ensure that, with high probability, the finite GP sample set contains at least one member uniformly within $\zeta$ of this fixed $f^\star$, even though $f^\star$ itself need not be a draw from a GP. 
Consequently, enforcing the tightened expected-cost constraint on every sampled model ensures safety under the true dynamics:
\[
J_c(\pi_n,f^\star)\leq d,
\]
simultaneously for every training episode and for the returned policy.

Ultimately, \algname{SBSRL} integrates the GP close-sample certificate and sampled-model constraints into an episodic, finite-horizon CMDP algorithm that uses an epistemic-exploration constraint to drive learning. Crucially, we prove the combined end-to-end result of expected-cost safety throughout learning, finite-time termination, and $\varepsilon$-near-optimality over the reachable component of the tightened true-safe policy set. These results extend beyond scenario approach optimization.

\subsection{Definition of Safety}
Our notion of safety is the standard finite-horizon CMDP constraint~\citep{altman}. For a policy $\pi$ and dynamics $f$, the cost return $J_c(\pi,f)$ in \Cref{eq: return} is the expected sum of costs along a trajectory, where the expectation is over the initial state and process noise. Safety during learning means that every policy $\pi_n$ selected for execution satisfies
\[
    J_c(\pi_n,f^\star)\leq d .
\]
This is an \emph{expected accumulated-cost} constraint, not a guarantee that every realized trajectory has cost at most $d$. 
Of course, expected cost needs averaging infinitely many trajectories, so it can only be estimated in practice. Accordingly, in the experiments we estimate $J_c$ by averaging accumulated costs over finitely many separate evaluation trajectories at each training episode; this is denoted as $\hat J_c$ in our plots.

The high-probability statement in the safety theorem of \Cref{section:results} plays a distinct role. Specifically, it guarantees that, with probability at least $1-2\delta$ over the randomness of learning and model sampling, the expected-cost constraint above holds simultaneously for all selected episodes (and for the returned policy).

\subsection{Kernels and Unbounded Domains}
We first note that the regularity assumptions on the kernel required by \Cref{assumption:continuity_sigma,assump:information_gain} are not restrictive, and can hold despite the fact that the state-action space is unbounded.

In particular, suppose $k(z,z')=\Tilde{k}(\phi(z),\phi(z'))$, where $\phi\colon\,\calZ\to\Tilde{\calZ}\subset\mathbb{R}^{d_{\Tilde{z}}}$ is a bounded and Lipschitz feature map, and $\Tilde{k}$ is Lipschitz on $\Tilde{\calZ}$ (which holds for common kernels such as linear, squared exponential, and Matérn since $\Tilde{\calZ}$ is bounded). Then $k$ is Lipschitz on $\calZ$ (thus satisfying Assumption \ref{assumption:continuity_sigma}), even if $\calZ$ is unbounded. Moreover, since the information gain of $k$ on the inputs $\{z_i\}$ is equal to the information gain of $\Tilde{k}$ on the transformed inputs $\{\phi(z_i)\}$, the maximum information gain $\gamma_{N}(k)$ over $\calZ$ is upper bounded by the maximum information gain of $\gamma_N(\Tilde{k})$ over the bounded domain $\tilde Z$. 

Therefore, for common choices of $\Tilde{k}$, such as the squared exponential or linear kernel, $\gamma_{N}(k)$ grows poly-logarithmically in $N$ \citep{Srinivas_2012}, and Assumption \ref{assump:information_gain} is also satisfied. For Matérn kernels, the information gain grows polynomially in $N$ with growth determined by a smoothness parameter, and satisfies Assumption~\ref{assump:information_gain} for sufficiently large choices of the smoothness parameter.

We next provide an overview of possible methods to sample a function realization from a GP in practice, followed by a discussion on alternative design choices for model updates.

\subsection{Sampling from a GP}
In practice, the sampling of continuous functions $\tilde f\sim GP(\mu,k)$ cannot be implemented directly due to the infinite dimensional feature space of  GPs. 

However, we are only interested in the function values at the locations in the state-action space reached by the agent across all episodes, i.e. $\{z_t^n\}_{t=0,\dots,T-1}^{n\geq0} \subset \calZ$.
Therefore, for each scalar sample $\tilde f_j$, which represents a particular deterministic realization in function space, we are technically interested in its finite-dimensional marginals: $(\tilde f_j(z_t^n))_{t,n} \sim \mathcal{N}([\mu(z_t^n)]_{t,n}, [k(z_t^n, z_{t'}^{n'})]_{t,n,t',n'})$. 

However, in a sequential control setting, the locations at later timesteps and episodes are not known in advance, preventing us from constructing this joint distribution upfront. 
Crucially, Property $1$ in \citet{beckers:TAC2021} shows that this is equivalent to iteratively sampling each function value by conditioning the GP on the sample's own predictions at all previously visited locations, treating them as noise-free training points. This method, known as sequential pathwise conditioning or forward sampling, ensures the spatial consistency of the sampled function across episodes. This has been effectively applied in both scenario-based control \citep{2018ecc..conf..228U} and MPC settings \citep{prajapat2024towards,prajapat2025finite}.

For completeness, other methods can be used to draw approximate samples from a GP \citep{Garnett_2023,Do2025SamplingFG}. For instance, one could construct a spectral function approximation of the GP (e.g., via Random Fourier Features) and draw a finite-dimensional weight vector from the prior. Although this yields an approximation rather than an exact draw, it defines a global function that can be trivially evaluated at any desired location, making it highly popular in practice due to its computational efficiency.

\subsection{Updating the GP Samples}
For the sake of exposition in Section~\ref{section:algorithm}, we formulated our theory by drawing a finite set of samples from the prior GP, purposefully not conditioning the underlying function realization on the true data gathered during exploration. To embed the newly collected information, we instead only truncate the prior samples to lie within the confidence region established by Lemma~\ref{lemma:confidence_gp}. This is conceptually similar to \citet{prajapat2024towards}, who enforce truncation by directly discarding all samples violating the confidence bounds. 

A natural alternative to incorporate newly collected data is to occasionally discard the prior samples and draw new ones from the posterior GP.
Since the posterior generally provides a tighter approximation of the true dynamics than the prior, resampling could potentially reduce the number of samples $M$ required to obtain a $\zeta$-close model via Lemma~\ref{lem: sampling}. 
However, repeated resampling introduces a more complex analytical structure, as the sampled models would no longer be coupled across episodes.

More precisely, while \Cref{lem: sampling} is stated for a single set of prior samples, Theorem 1 of \citet{prajapat2025finite} allows us to account for a GP conditioned on observed data. 
The only technical distinction is that the lower bound on the number of samples $M_n$ must be computed from the history observed  before sampling at episode $n$, and is therefore not deterministic  before learning. Let $\{\delta_n\}_{n\geq0}$ be any positive sequence satisfying $\sum_{n=0}^{\infty}\delta_n\leq\delta$, for example $\delta_n=6\delta/(\pi^2(n+1)^2)$. Conditional on the observed history, \citet{prajapat2025finite} give a finite sample count $M_n$ for which the posterior sample set contains a $\zeta$-close path except with probability $\delta_n$. A union bound over all episodes therefore guarantees the existence of such a sample at every episode with probability at least $1-\delta$.
Since every posterior sample is first clipped to the initial envelope $\mathcal F_0^B$, \Cref{assump:prior_knowledge} still ensures that the sampled safety constraints remain feasible, so the algorithm cannot terminate earlier because of them.

We further remark that the episodic truncation is only necessary for our optimality result (specifically the second claim of \Cref{thm:merged}). To guarantee safety, we only need to satisfy the lower-bound on the number of samples in Lemma~\ref{lem: sampling} (provided Assumption~\ref{assump:prior_knowledge} holds, for which is sufficient to perform a truncation with respect to the prior bound $B$).

Appendix~\ref{appendix:experiments} provides an ablation comparing the different sampling strategies discussed here. 

\subsection{Feasibility of the Safety Constraint}
One may wonder whether the algorithm could terminate because the safety constraint becomes infeasible at some earlier episode, rather than due to the exploration constraint. 
Remarkably, under the success event of \Cref{lemma:confidence_gp}, \Cref{assump:prior_knowledge} ensures that the sampled safe set $\Pi^n_{\mathrm{safe}}$ is non-empty for every episode $n\geq0$. Consequently, under this event, infeasibility can only occur due to the exploration constraint.  

More specifically, conditioning on the success event of Lemma \ref{lemma:confidence_gp}, the truncation procedure is well-defined and the sequence of confidence sets containing the truncated samples is nested (i.e. non-increasing). Combined with the truncation defined for the initial episode $n=0$, this implies that for all $n\geq0$ and all $m\in[M]$ the samples $f_n^m$ lie in the intersection of the successive confidence sets, and in particular remain within the initial truncation envelope $\calF_0^B$. Thus, by Assumption \ref{assump:prior_knowledge} and the fact that $\Delta_\zeta$ is chosen to satisfy $\Delta_{\zeta} \leq \Delta$, any policy $\pi\in\Pi_\text{prior}$ must also satisfy
        $J_c(\pi, f_n^m) \le d - \Delta < d-\Delta_\zeta$  for $m=1,\dots,M$.
    This means that $\Pi_\text{prior}\subseteq\Pi^n_{\mathrm{safe}}$ and thus $\Pi^n_{\mathrm{safe}}$ is non-empty for all $n\geq0$ by Assumption \ref{assump:prior_knowledge}. Therefore, the safety constraint remains feasible throughout learning, and termination can only occur due to infeasibility of the exploration constraint.

\section{Technical Lemmas} \label{appendix:technical}
In this section, we state and prove some auxiliary lemmas that will be useful in the proofs of our main results. We start with the proof of \Cref{lem: sampling}, which was already stated in \Cref{section:background}.
\samplesthm*
\begin{proof}
    Lemma \ref{lem: sampling} is a direct adaptation of Theorem 1 from \citet{prajapat2025finite} to our specific setting. To bridge the two results, we apply a centering transformation, defining the shifted quantities $\bar{f}^m \coloneqq \tilde f^m - \mu$ and $\bar{f}^\star \coloneqq f^\star - \mu$. Because $\tilde f^m \sim GP(\mu, k)$, the centered sample follows a zero-mean Gaussian process, $\bar{f}^m \sim GP(0,k)$. Similarly, under \Cref{assump:rkhs}, the centered ground truth satisfies $\|\bar{f}^\star_j\|_k \leq B$ for all $j \in [d_x]$. Thus, the conditions for Theorem 1 of \citet{prajapat2025finite} are met, and establishing $\|\bar{f}_j^m - \bar{f}_j^\star\|_\infty \leq \zeta  \;\;\forall j\in[d_x]$ directly implies $\|\tilde f_j^m - f_j^\star\|_\infty \leq \zeta \;\;\forall j\in[d_x]$.
    
    Our formulation simplifies the original result in two key ways. First, by applying the theorem to the prior GP directly rather than on the posterior conditioned on data, we eliminate the reliance on data-dependent quantities, effectively simplifying their confidence parameter to $C_D = B^2/2$. Second, because \Cref{assump:rkhs} imposes a deterministic bound $B$ on the RKHS norm of the mismatch, we bypass the need for the probabilistic union bound used in the original paper, allowing our sample complexity to scale with $\log(\delta)$ rather than $\log(\delta/2)$. Finally, we employ the vector-valued extension discussed in Remark 1 of \citet{prajapat2025finite}. To ensure the $\infty$-norm bound holds across all $d_x$ independent state coordinates simultaneously, the marginal probabilities must be multiplied. This introduces the factor of $d_x$ that multiplies the exponent $B^2/2 + \phi(\zeta)$ in Lemma \ref{lem: sampling}.
\end{proof}

We next present a lemma that generalizes the notation of Lemma B.3 in \citet{sukhija2025sombrl}, Lemma A.6 of \citet{as2025actsafe} or Corollary 3 of \citet{OptiActive}. Together with the subsequent corollary, it provides bounds on the discrepancy between returns under two different dynamics models, which are used in the proofs of safety and optimality.

\begin{lemma}\label{lemma:bound_J_general}
Let Assumptions \ref{assump:process_noise}-\ref{assump:rkhs} hold.
Consider any positive and bounded function $g \in [0, G_{\max}]$. Let $f$ and $f'$ be any pair of dynamics functions.
Then, for all $n \ge 0$
\begin{equation*}
\big| J_g(\pi,f) - J_g(\pi,f') \big|
\le 
T G_{\max} \, \mathbb{E}_{\tau_\pi^f}\left[ \sum_{t=0}^{T-1}
\frac{\| f'(x_t,\pi(x_t)) - f(x_t,\pi(x_t)) \|}{\sigma_w}\right]
\end{equation*}
where the expectation is rolled out over $f$, i.e. $x_{t+1}=f(x_t,\pi(x_t))+w_t$. An analogous bound holds in terms of an expectation over $f'$ with $x_{t+1}'=f'(x_t',\pi(x_t'))+w_t$: 
\begin{equation*}
\big| J_g(\pi,f) - J_g(\pi,f') \big|
\le 
T G_{\max} \, \mathbb{E}_{\tau_\pi^{f'}}\left[ \sum_{t=0}^{T-1}
\frac{\| f'(x_t',\pi(x_t')) - f(x_t',\pi(x_t')) \|}{\sigma_w}\right].
\end{equation*}
\end{lemma}

\begin{proof}
We give proof for the first inequality. The same argument holds for the second inequality by swapping $f$ and $f'$. Let $J_{g,t+1}(\pi,f,x)$ denote the cost-to-go from state $x$ at step $t+1$ under some dynamics $f$ and policy $\pi$:
\begin{equation*}
    J_{g,t+1}(\pi,f,x)=\mathbb{E}_{\tau_\pi^f}\left[ \sum_{k=t+1}^{T-1}
g(x_k,\pi(x_k))\right]
\end{equation*}

We can directly apply Corollary 2 of \citet{OptiActive} to obtain
\begin{equation*}
J_g(\pi,f') - J_g(\pi,f)
=
\mathbb{E}_{\tau_\pi^f}\left[
\sum_{t=0}^{T-1}
\Big(
J_{g,t+1}(\pi,f',\Tilde{x}_{t+1}) 
-
J_{g,t+1}(\pi,f',x_{t+1})
\Big)\right]
\end{equation*}
where the expectation is w.r.t.\ $\pi$ under the dynamics $f$, and
\begin{align*}
x_{t+1} &= f(x_t, \pi(x_t)) + w_t, \\
\Tilde{x}_{t+1} &= f'(x_t, \pi(x_t)) + w_t.
\end{align*}
Therefore,
\begin{align*}
|J_g(\pi,f') - J_g(\pi,f)|
&=\vert\mathbb{E}_{\tau_\pi^f}\left[
\sum_{t=0}^{T-1}
\Big(
J_{g,t+1}(\pi,f',\Tilde{x}_{t+1}) 
-
J_{g,t+1}(\pi,f',x_{t+1})
\Big)\right]\vert \\
&\le
\sum_{t=0}^{T-1}
\mathbb{E}_{\tau_\pi^f}\!\left[|
\mathbb{E}_{w_t}\big[
J_{g,t+1}(\pi,f',\Tilde{x}_{t+1}) 
-
J_{g,t+1}(\pi,f',x_{t+1})
\mid x_t \big]|
\right]
\end{align*}

Crucially, since $w_t\sim\mathcal{N}(0,\sigma_w^2)$ (Assumption \ref{assump:process_noise}), we have that $x_{t+1}$ and $\Tilde{x}_{t+1}$ are also gaussian given $x_t$, so that we can leverage Lemma C.2 of \citet{KakadeInfo}.
Furthermore, $J_{g,t+1}(\pi,f,x) \in [0, T G_{\max}]$ for all $\pi, f, x,$ and $t$.
Therefore, given $x_t$,
\begin{align*}
&|\mathbb{E}_{w_t}
\big[
J_{g,t+1}(\pi,f',\Tilde{x}_{t+1})
-
J_{g,t+1}(\pi,f',x_{t+1})
\mid x_t \big]|\\
&\le
\max\left\{ \sqrt{
\mathbb{E}_{w_t}[J_{g,t+1}^2(\pi,f',\Tilde{x}_{t+1}) \mid x_t]},\;
\sqrt{\mathbb{E}_{w_t}[J_{g,t+1}^2(\pi,f',x_{t+1})\mid x_t]}\right\}\\
&\times \min
\left\{
\frac{\| f'(x_t,\pi(x_t)) - f(x_t,\pi(x_t))\|}{\sigma_w}, 1
\right\} \text{ (\citet{KakadeInfo}, Lemma C.2)}\\
&\le
T G_{\max}
\min\left\{
\frac{\| f'(x_t,\pi(x_t)) - f(x_t,\pi(x_t)) \|}{\sigma_w},\; 1
\right\}.
\end{align*}
Summing over $t$ concludes the proof.
\end{proof}

Finally, in the case when one of the dynamics functions represents exactly the mean dynamics, i.e. $f'=\mu_n$, we can continue from the above lemma to give a bound in terms of the predicted uncertainty. This is formalized in the following corollary. We first define $\calF_n$ to be the filtration that contains all the data observed through episode $n$ (i.e. $\calD_{0:n}$) and the the realizations of the samples obtained at the initial episode (i.e. $\tilde f^m\; \forall m\in[M]$). Then, overloading the definition of the return in Section~\ref{section:problem_setting}, we can define
\begin{align*}
    J_{s_n}(\pi,f)\coloneqq\mathbb{E}_{\tau_\pi^f}\left[ \sum_{t=0}^{T-1}{s_n(x_t, \pi(x_t))}\mid \calF_{n-1}\right],
\end{align*}
where we note that the random variable $s_n$ is $\calF_{n-1}$ measurable.

\begin{corollary}[Adapted from Lemma B.3 of \citet{sukhija2025sombrl} or Lemma A.6 of \citet{as2025actsafe}]\label{corollary:bound_uncertainty_general}
    Let Assumptions \ref{assump:process_noise}-\ref{assump:rkhs} hold. Consider any positive and bounded function $g \in [0, G_{\max}]$ and fix any policy $\pi\in\Pi$. Consider the GP fit after episode $n$, characterized by $(\mu_n,\sigma_n)$.
    Let $f$ be any dynamics function satisfying for all $x\in\mathcal{X}$ and all $j\in [d_x]$:
\begin{align*}
    |\mu_{n,j}(x,\pi(x)) - f_j(x,\pi(x))| \leq \beta_n(\delta) \sigma_{n,j}(x,\pi(x)). 
\end{align*}
where $\beta_n$ is the one defined in Lemma \ref{lemma:confidence_gp}. Then,
\begin{equation*}
\big| J_g(\pi,f) - J_g(\pi,\mu_n) \big|
\le 
\lambda_n^g\,\min \left\{J_{s_n}(\pi,f),J_{s_n}(\pi,\mu_n)\right\} 
\end{equation*}
where $\lambda_n^g\coloneqq G_{\max}T \frac{\beta_n(\delta)}{\sigma_w}$.

\end{corollary}
\begin{proof}
    We give the proof for the first inequality, while the second inequality follows analogously after swapping $f$ and $\mu_n$. \\
    By applying Lemma \ref{lemma:bound_J_general}, we obtain
    \begin{equation*}
\big| J_g(\pi,f) - J_g(\pi,\mu_n) \big|
\le 
T G_{\max} \, \mathbb{E}_{\tau_\pi^f}\left[ \sum_{t=0}^{T-1}
\frac{\| f(x_t,\pi(x_t)) - \mu_n(x_t,\pi(x_t)) \|}{\sigma_w}\right]
\end{equation*}
Then, by our assumption on $f$ lying in the confidence interval defined by $(\mu_n,\sigma_n, \beta_n)$, we can bound:
\begin{equation*}
   \| f(x_t,\pi(x_t)) - \mu_n(x_t,\pi(x_t)) \|\leq \beta_n(\delta)\| \sigma_n(x_t,\pi(x_t)) \|
\end{equation*}
Combining this with the above bound finally proves the claim:
\begin{align*}
\big| J_g(\pi,f) - J_g(\pi,\mu_n) \big|
&\le T G_{\max} \, \mathbb{E}_{\tau_\pi^f}\left[ \sum_{t=0}^{T-1}
\frac{\| f(x_t,\pi(x_t)) - \mu_n(x_t,\pi(x_t)) \|}{\sigma_w}\right]\\
&\leq G_{\max}T \frac{\beta_n(\delta)}{\sigma_w} \mathbb{E}_{\tau_\pi^f}\left[ \sum_{t=0}^{T-1}{s_n(x_t, \pi(x_t))}\mid \calF_{n-1}\right] \eqqcolon \lambda_n^g\,J_{s_n}(\pi,f).
\end{align*}

\end{proof}

\section{Proof of Safety}\label{appendix:safety}

\safetythm*

\begin{proof}

Lemma \ref{lemma:bound_J_general} allows to bound the difference in expected cost between the trajectory sampled from the true dynamics $f^{\star}$ and the one sampled from a truncated sample $f_{n}^m$, for any $m\in[M]$, obtained according to the \textit{same policy $\pi$ and the same initial distribution $x_0$}:
\begin{equation} \label{eq:bound_cost1}
\big| J_c(\pi, f_{n}^m) - J_c(\pi,f^{\star}) \big|
\le 
T C_{\max} \, \mathbb{E}_{\tau_\pi^f}\left[ \sum_{t=0}^{T-1}
\frac{\| f^{\star}(x_t,\pi(x_t)) -  f_{n}^m(x_t,\pi(x_t)) \|}{\sigma_w}\right]
\end{equation}

    Then, as described in Section \ref{section:algorithm}, the number of samples $M$ in Algorithm \ref{alg:mbrl1} is chosen large enough to ensure that with probability $1-\delta$ there exists a sample $f_0^{m_\zeta}$ from the GP prior that is $\zeta-$ close to $f^\star$. Furthermore, under the success event of Lemma \ref{lemma:confidence_gp}, which holds with probability $1-\delta$, the truncation that we perform at every episode can only bring the samples closer to $f^\star$, preserving the existence of a $\zeta$-close dynamics sample. Thus, by taking a union bound over the above events, we can state that with probability at least $1-2\delta$, there exists a $\zeta$-close truncated sample $ f_n^{m_\zeta}$ at any episode $n$:
\begin{equation*}
\left\lVert f^{\star}(z) - f_n^{m_\zeta}(z) \right\rVert 
< \,\zeta\sqrt{d_x} \;\; \forall z\in\calZ
\end{equation*}
where the factor $\sqrt{d_x}$ arises from bounding the $\ell_2$ norm above in terms of the $\ell_\infty$ norm of \Cref{lem: sampling}.
Thus, conditioned on the existence of the $\zeta$-close sample, we can substitute it in (\ref{eq:bound_cost1}) to bound the difference in expected cost between the trajectory sampled from the true dynamics $f^{\star}$ and the $\zeta$-close sample $ f_n^{m_\zeta}$ under the same policy $\pi$ and the same initial distribution $x_0$:
\begin{align*}
\big| J_c(\pi, f_n^{m_\zeta}) - J_c(\pi,f^{\star}) \big|
&\le 
T C_{\max} \, \mathbb{E}_{\tau_\pi^f}\left[ \sum_{t=0}^{T-1}
\frac{\| f^{\star}(x_t,\pi(x_t)) -f_n^{m_\zeta}(x_t,\pi(x_t)) \|}{\sigma_w}\right]\\
&\le\zeta\sqrt{d_x}\frac{T^2 C_{\max}}{\sigma_w}\eqqcolon \Delta_\zeta.
\end{align*}
This upper bound matches exactly the definition of the tightening $\Delta_\zeta$ that we use when enforcing the constraints in Problem~(\ref{pb:1}).

We are finally ready to show the desired safety bound. By definition, any policy $\pi\in \Pi_{\textrm{safe}}^n$  satisfies the tightened constraints for all truncated samples, i.e. $ \quad J_c(\pi,  f_n^m) \le d - \Delta_\zeta, \quad \forall m \in M$. However, as shown above, with probability at least $1-2\delta$ there exists a $\zeta$-close sample $ f_n^{m_\zeta}$ that satisfies $\left|J_c\left(\pi, f^{\star}\right) - J_c\left(\pi, f_n^{m_\zeta}\right)\right|\leq \Delta_\zeta$. Rearranging, we obtain
\begin{equation}
    J_c\left(\pi, f^{\star}\right)\leq J_c\left(\pi,  f_n^{m_\zeta}\right) + \Delta_\zeta \leq d - \Delta_\zeta + \Delta_\zeta = d,
\end{equation}
which proves the claim.

\end{proof}
In the following, we provide the proof of \Cref{thm:merged}. For the sake of exposition, we separate the two claims of the theorem, proving the first claim (about finite-time termination of \Cref{alg:mbrl1}) in Appendix~\ref{appendix:sample_complexity}, and the second claim (about optimality of the policy returned by \Cref{alg:mbrl1}) in Appendix~\ref{appendix:optimality}.
\section{Proof of Sample Complexity}\label{appendix:sample_complexity}
\begin{lemma}\label{lemma:martingale_concentration}
Consider running \Cref{alg:mbrl1} in the setting of \Cref{thm:merged}, until some finite episode $\bar{n}$. Then, with probability at least $1-\delta$, we have:
\begin{align}\label{eq:lhs}
    \sum_{n=0}^{\bar{n}-1} 
\mathbb{E}_{\tau_{\pi_n}^{f^\star}}\!\left[
\sum_{t=0}^{T-1} s_n(x_t,\pi_n(x_t))
\mid \calF_{n-1}\right]\leq\sum_{n=0}^{\bar{n}-1} \sum_{t=0}^{T-1} s_n(x_t^n,\pi_n(x_t^n)) + T\sqrt{2\bar{n}d_x\sigma_{\max} \log(1/\delta)},
\end{align}
where, for each episode $n\in\{0,\dots,\bar n-1\}$, the states $x_t^n$ on the right-hand side denote the trajectory realized by the agent in the environment, while following the policy $\pi_n$ under the true dynamics $f^\star$.

\end{lemma}
\begin{proof}
    Let us define $Y_n\coloneqq \sum_{t=0}^{T-1} s_n(x_t^n,\pi_n(x_t^n))$, so that:
    \begin{equation}\label{eq:expected_var}
         \mathbb{E}_{\tau_{\pi_n}^{f^\star}}\!\left[
\sum_{t=0}^{T-1} s_n(x_t,\pi_n(x_t))
\mid \calF_{n-1}\right]=  
\mathbb{E}_{\tau_{\pi_n}^{f^\star}}\!\left[
Y_n\mid \calF_{n-1}
\right].
    \end{equation}
Next, define
\begin{align}
    X_n\coloneqq \mathbb{E}_{\tau_{\pi_n}^{f^\star}}\!\left[Y_n\mid\calF_{n-1}\right] - Y_n
\end{align}
Then, $X_n$ is a \emph{martingale difference sequence}, because:
\begin{enumerate}
    \item $\mathbb{E}_{\tau_{\pi_n}^{f^\star}}\!\left[X_n\mid\calF_{n-1}\right]=0$.
    \item $\mid X_n\mid\leq T\sqrt{d_x\sigma_{\max}}$ (because $Y_n\in\left[0,\,T\sqrt{d_x\sigma_{\max}}\right]$).
\end{enumerate}
Therefore, we can apply Lemma A.7 of \citet{bookCesa} \footnotemark (substituting $A_i\coloneqq-T\sqrt{d_x\sigma_{\max}}$ and $c_i\coloneqq2T\sqrt{d_x\sigma_{\max}}$ in their notation) to get, with probability at least $1-\delta$:
\begin{align}
    \sum_{n=0}^{\bar{n}-1}X_n\leq T\sqrt{2\bar{n}d_x\sigma_{\max}\log(1/\delta)}
\end{align}
The claim follows by rearranging and using the definition of $X_n$, $Y_n$, and \Cref{eq:expected_var}.
\footnotetext{Note that a variance-sensitive inequality (e.g. Freedman or Bernstein inequalities for martingales) would yield a sharper bound, but Azuma-Hoeffding suffices for our purposes and yields a cleaner proof.}
\end{proof}

Next, we restate the first claim of \Cref{thm:merged} as follows.
\begin{restatable}[Sample Complexity]{theorem}{samplecomplexitythm} \label{thm:complexity}
Suppose Assumptions \ref{assump:process_noise}-\ref{assump:information_gain} hold. Fix some $\delta\in(0,1/2)$, $\zeta \in(0,\frac{\sigma_w\Delta}{\sqrt{d_x}T^2C_{\max}})$, and $\varepsilon > 0$.
Let $d_\sigma^n=\frac{\varepsilon\sigma_w}{2G_{\text{max}} T\beta_n(\delta)}$ with $G_{\text{max}}\coloneqq\max\{C_\text{max},R_\text{max}\}$, and consider running \Cref{alg:mbrl1}.
There exists a problem-dependent constant $C = C(d_x, \sigma_{\max}, G_{\max}, T, \delta)$ such that, letting $\bar n$ be the smallest integer satisfying 
\begin{align}\label{eq: sample complexity}
\bar n \;\ge\; C \, \frac{\gamma_{\bar n T}(k)^3}{\varepsilon^2},
\end{align}
then, with probability at least $1-2\delta$ there exists $n^\star \leq \bar n$ such that
    $J_{s_{n^\star}}(\pi, \mu_{n^\star})  < d_\sigma^{n^\star} \,\, \forall \pi \in\Pi^{{n^\star}}_{\text{safe}},$
    i.e. Algorithm~\ref{alg:mbrl1} terminates at iteration ${n}^\star$.
\end{restatable}

\begin{proof}
We start by conditioning on the success event of Lemma \ref{lemma:confidence_gp}, i.e. we assume well-calibration of the true dynamics $f^\star$. Under this event, that holds with probability $1-\delta$, and further noting that $s_n(\cdot)=\|\sigma_n(\cdot, \pi(\cdot))\|$ is bounded by $\sqrt{d_x\sigma_{\max}}$ thanks to Assumption \ref{assump:rkhs}, we can apply Corollary \ref{corollary:bound_uncertainty_general} to the scalar function $s_n(\cdot)$ and dynamics $f^\star$, for any $n\geq0$. Thus, we can write for $J_{s_n}$:
\begin{equation*}
    \mid J_{s_n}(\pi_n,\mu_n) - J_{s_n}(\pi_n, f^{\star}) \mid \leq \lambda_n^{s} J_{s_n}(\pi_n,f^{\star}),
\end{equation*}
where  $\lambda_n^{s} = \sqrt{d_x\sigma_{\max}} T\frac{\beta_n(\delta)}{\sigma_w}$.
Rearranging the terms:
\begin{align*}
   J_{s_n}(\pi_n,\mu_n)&\leq \left(1 + \lambda_n^{s} \right)J_{s_n}(\pi_n,f^{\star}).
\end{align*}
Notice that by our uncertainty constraint we know that before termination $J_{s_n}(\pi_n,\mu_n)\geq d_\sigma^n$. Furthermore, using the monotonicity of $\beta_n(\delta)$ we can bound $d_\sigma^n\geq d_\sigma^{\Bar{n}}$ for all $n\leq\Bar{n}$. By substituting this and taking the sum over the episodes from $1$ to $\Bar{n}$:
\begin{align*}
   \Bar{n}d_\sigma^{\Bar{n}} \leq \sum_{n=0}^{\bar{n}-1} d_\sigma^{n} &\leq \sum_{n=0}^{\bar{n}-1}J_{s_n}(\pi_n,\mu_n) \\
   &\leq \sum_{n=0}^{\bar{n}-1}\left(1 + \lambda_n^{s} \right)J_{s_n}(\pi_n,f^{\star}) \\
   &= \sum_{n=0}^{\bar{n}-1}\left(1 + \lambda_n^{s} \right)\mathbb{E}_{\tau_{\pi_n}^{f^\star}}\!\left[\sum_{t=0}^{T-1} s_n(x_t,\pi_n(x_t))\mid \calF_{n-1}\right]\\
   &\leq \left(1 + \lambda_{\bar n}^s \right)\sum_{n=0}^{\bar{n}-1}\mathbb{E}_{\tau_{\pi_n}^{f^\star}}\!\left[\sum_{t=0}^{T-1} s_n(x_t,\pi_n(x_t)) \mid \calF_{n-1}\right] \\
   &\leq \left(1 + \lambda_{\bar n}^s \right)\left(\sum_{n=0}^{\bar{n}-1}\sum_{t=0}^{T-1} s_n(x_t^n,\pi_n(x_t^n))+ T\sqrt{2\bar{n}d_x\sigma_{\max}\log(1/\delta)}\right) \;\;\text{(\Cref{lemma:martingale_concentration})},\\
   &\leq\left(1 + \lambda_{\bar n}^s \right)\left(\sqrt{\bar{n}T}\sqrt{\sum_{n=0}^{\bar{n}-1}\sum_{t=0}^{T-1} s_n(x_t^n,\pi_n(x_t^n))^2}+ T\sqrt{2\bar{n}d_x\sigma_{\max}\log(1/\delta)}\right) \\
   &\;\;\;\;\;\;\;\;\;\;\;\;\;\;\;\;\;\;\;\;\;\text{(Cauchy-Schwarz),}\\
   &\leq\left(1 + \lambda_{\bar n}^s \right)\left(\sqrt{\bar{n}T}\sqrt{\frac{T d_x \sigma_{\max}}{\log(1+\sigma_w^{-2}\sigma_{\max})}
\, \gamma_{\bar{n}T}}+ T\sqrt{2\bar{n}d_x\sigma_{\max}\log(1/\delta)}\right)  \\
   &\;\;\;\;\;\;\;\;\;\;\;\;\;\;\;\;\;\;\;\;\; \text{(\citet{HURCL}, Lemma 17)},
\end{align*}
where we used that $\lambda_n^{s}\leq\lambda_{\bar n}^s$ due to the monotonicity of $\beta_n(\delta)$. Note that the success event of  \Cref{lemma:martingale_concentration} holds with probability probability at least $1-\delta$, so that taking a union bound with the success event of Lemma~\ref{lemma:confidence_gp} we get that our final guarantee holds with probability at least $1-2\delta$.
Thus, dividing both sides for $\sqrt{\Bar{n}}$ and taking the square,  we obtain:
\begin{align*}
   \Bar{n} &\leq \frac{T^2\left(1 + \lambda_{\bar n}^s \right)^2}{(d_\sigma^{\Bar{n}})^2}\left(\sqrt{\frac{d_x \sigma_{\max}}{\log(1+\sigma_w^{-2}\sigma_{\max})}\, \gamma_{\bar{n}T}}
+ \sqrt{2d_x\sigma_{\max}\log(1/\delta)}\right)^2 \\
   &\leq\frac{4T^4\left(1 + \lambda_{\bar n}^s \right)^2G_{\text{max}}^2}{\varepsilon^2\sigma_w^2}\beta_{\Bar{n}}^2(\delta)\left(\sqrt{\frac{d_x \sigma_{\max}}{\log(1+\sigma_w^{-2}\sigma_{\max})}\, \gamma_{\bar{n}T}}
+ \sqrt{2d_x\sigma_{\max}\log(1/\delta)}\right)^2 
.
\end{align*}
Therefore, if there exists an integer $\bar n$ that violates the above inequality, the algorithm must have terminated in one of the previous rounds, which proves the claim. To show that such an integer exists, note that the calibration coefficient $\beta_n(\delta)$ is monotonically increasing in $n$. Therefore, a finite integer $\bar n$ violating the above inequality (i.e. satisfying Equation \ref{eq: sample complexity}) can be found if $\beta_n(\delta)^4 \gamma_{nT}(k)$ grows sub-linearly with $n$ (since asymptotically we can omit the low order terms in $\gamma_{nT}$). By the definition of $\beta_n(\delta)$ from \Cref{lemma:confidence_gp}, the quantity $\beta_n^4 (\delta)\gamma_{nT}(k)$ is of the same order as $\gamma_{nT}^3(k)$ for large $n$. Therefore, the existence of a finite $\bar n$ satisfying Equation~\eqref{eq: sample complexity} is guaranteed by Assumption \ref{assump:information_gain}.
\end{proof}

\section{Proof of Optimality}\label{appendix:optimality}

\begin{lemma}\label{lemma:specialized_bound_mean_uncertainty}
    Let Assumptions \ref{assump:process_noise}-\ref{assump:rkhs} hold. 
Then, with probability at least $1-\delta$ we have that for all $n \ge 0$ and any truncated sample $ f_{n}^m$: 
\begin{equation}
    |J_c(\pi,f^{\star}) - J_c(\pi, f_{n}^m)| \leq 2\lambda_n^cJ_{s_n}(\pi,\mu_n).
\end{equation}
Moreover, the same holds for the reward function, i.e. replacing $c$ and $\lambda_n^c$ with $r$ and $\lambda_n^r$ respectively.
\end{lemma}
\begin{proof}
Under the success event of Lemma \ref{lemma:confidence_gp}, which holds with probability $1-\delta$, and thanks to the truncation procedure described in Section \ref{section:algorithm}, we can apply Corollary \ref{corollary:bound_uncertainty_general} to bound the distance between the costs attained along the mean dynamics and the ones attained along the trajectory rolled out using any of the samples $ f_{n}^m$, in terms of the uncertainty predicted along the mean dynamics. Thus, by Corollary \ref{corollary:bound_uncertainty_general} we immediately get
\begin{equation} \label{bound_cost1}
    |J_c(\pi, f_{n}^m) - J_c(\pi,\mu_n)| \leq \lambda_n^cJ_{s_n}(\pi,\mu_n).
\end{equation}
We can then repeat the same argument for the true dynamics, i.e. swapping $ f_{n}^m$ with $f^\star$:
\begin{equation}\label{bound_cost2}
    |J_c(\pi,f^{\star}) - J_c(\pi,\mu_n)| \leq \lambda_n^cJ_{s_n}(\pi,\mu_n).
\end{equation}

Finally, the claim follows by combining the two bounds above.
\end{proof}

We are finally ready to prove the second claim of \Cref{thm:merged}, which can be restated as follows.
\begin{restatable}[]{theorem}{optimalitythm}\label{thm:optimality}
Suppose Assumptions \ref{assump:process_noise}-\ref{assump:information_gain} hold. Fix some $\delta\in(0,1/2)$, $\zeta \in(0,\frac{\sigma_w\Delta}{\sqrt{d_x}T^2C_{\max}})$, $\varepsilon > 0$ and $\xi>\varepsilon+ 
\zeta\sqrt{d_x} 
\frac{T^2 C_{\max}}{\sigma_w}$.
Let $d_\sigma^n=\frac{\varepsilon\sigma_w}{2G_{\text{max}} T\beta_n(\delta)}$ with $G_{\text{max}}\coloneqq\max\{C_\text{max},R_\text{max}\}$, and consider running \Cref{alg:mbrl1}.
Then, with probability at least $1-2\delta$, the policy $\hat\pi$ returned by \Cref{alg:mbrl1} satisfies
    \begin{equation} \label{bound:optimality-proof}
    J_r(\Tilde{\pi}^{\star}, f^{\star}) - J_r(\hat \pi, f^\star)
    \leq \varepsilon,
    \end{equation}

where $\Tilde{\pi}^{\star}$ is an approximation of the optimal policy defined by the following optimization problem:
\begin{equation} \label{pb:pi_star}
\begin{aligned} 
  \Tilde{\pi}^\star &= \arg\max_{\pi \in \Pi_{\xi}^{\star,c}} J_r(\pi, f^\star). 
\end{aligned}
\end{equation}

\end{restatable}

\begin{proof} Let $n^\star$ denote the termination episode of \Cref{alg:mbrl1}, which is guaranteed to be finite by \Cref{thm:complexity}. 
First, Lemma \ref{lemma:specialized_bound_mean_uncertainty} allows us to calculate the difference between the expected cost along the real trajectory and the one generated from any truncated sample, in terms of the expected predicted uncertainty along the mean dynamics. Note that the bound holds jointly for all $n\geq0$ with probability at least $1-\delta$. Thus, conditioning on this event, we can write that for any policy $\pi\in\Pi$ at episode $n^\star$:
\begin{equation}\label{bound_cost:every_policy}
    |J_c(\pi,f^{\star}) - J_c(\pi, f_{n^\star}^m)| \leq 2\lambda_{n^\star}^cJ_{s_{n^\star}}(\pi,\mu_{n^\star}).
\end{equation}

Note that the bound above holds for any policy $\pi\in\Pi$. Then, if the policy also belongs to the set of safe policies (i.e. $\pi\in\Pi^{n^\star}_{\text{safe}}$) we have that it must be $J_{s_{n^\star}}(\pi,\mu_{n^\star})<d_\sigma^{n^\star}$, since we assume the algorithm has terminated due to the uncertainty constraint. However, recalling the definition of $d_{\sigma, n}$ and $\lambda_n^c$, we have that $d_{\sigma}^{n^\star}=\frac{\varepsilon\sigma_w}{2G_{\max} T\beta_{n^\star}(\delta)}$ and $\lambda_{n^\star}^c=C_{\max} T\frac{\, \beta_{n^\star}(\delta)}{\sigma_w}$. Thus, we can substitute these in (\ref{bound_cost:every_policy}) to get that for any policy $\pi\in\Pi^{n^\star}_{\text{safe}}$ and any truncated dynamics sample $ f_{n^\star}^m$,
\begin{equation} \label{bound_cost:final}
    \left|J_c\left(\pi, f^{\star}\right) - J_c\left(\pi,  f_{n^\star}^m\right)\right|\leq \varepsilon.
\end{equation}
In the following, we will also use that for any policy $\pi\in\Pi^{n^\star}_{\text{safe}}$ we have,
\begin{equation} \label{bound_reward:final}
    \left|J_r\left(\pi, f^{\star}\right) - J_r\left(\pi, \mu_{n^\star}\right)\right|\leq \frac{\varepsilon}{2},
\end{equation}
which can be obtained by repeating the same argument we used to obtain Equation \ref{bound_cost:final}, but applied directly to Equation \ref{bound_cost2} of Lemma \ref{lemma:specialized_bound_mean_uncertainty}, after substituting $c$ with $r$. 

We highlight that in general Equations \ref{bound_cost:final}-\ref{bound_reward:final} only hold for policies in $\Pi^{n^\star}_{\text{safe}}$, since Problem~(\ref{pb:1}) being infeasible only implies that the uncertainty is lower than $d_\sigma^{n^\star}$ on the policies that satisfy the first constraint (i.e. that belong to $\Pi^{n^\star}_{\text{safe}}$). Indeed, upon termination there can still exist some $\pi\notin \Pi^{n^\star}_{\text{safe}}$ that satisfies the second constraint  $J_{s_{n^\star}}(\pi, \mu_{n^\star})\geq d_\sigma^{n^\star}$, so that we cannot conclude that \eqref{bound_cost:final} holds for these policies.
 
However, Lemma \ref{lemma:pi_eps_subset_pi_safe} shows that upon termination of the algorithm it holds $ \Pi_{\xi}^{\star,c}\subseteq\Pi^{n^\star}_{\text{safe}}$.
 Thus, by the definition of $\Tilde{\pi}^{\star}\in \Pi_{\xi}^{\star,c}$ and Lemma \ref{lemma:pi_eps_subset_pi_safe}, we get that $\Tilde{\pi}^{\star}\in\Pi^{n^\star}_{\text{safe}}$.

This implies that $\Tilde{\pi}^\star$ is also a feasible solution of the problem obtained after removing the exploration constrain in Problem~(\ref{pb:1}), and hence it achieves a lower reward than $\hat \pi$:
\begin{equation} \label{eq:ineq_rewards}
    J_r\left(\Tilde{\pi}^\star,\mu_{n^\star} \right)\leq J_r\left(\hat \pi, \mu_{n^\star}\right).
\end{equation}
This allows to prove the following optimality bound that holds for the trajectory at the termination episode ${n^\star}$:
\begin{align*}
J_r\left(\Tilde{\pi}^\star, f^{\star}\right) - J_r\left(\hat \pi, f^{\star}\right) 
&= J_r\left(\Tilde{\pi}^\star, f^{\star}\right) - J_r\left(\Tilde{\pi}^\star, \mu_{n^\star}\right) \\
&\quad + J_r\left(\Tilde{\pi}^\star, \mu_{n^\star}\right) - J_r\left(\hat \pi, \mu_{n^\star}\right) \\
&\quad + J_r\left(\hat \pi, \mu_{n^\star}\right) - J_r\left(\hat \pi, f^{\star}\right) \\
&\leq \varepsilon,
\end{align*}
where we have used the bounds $|J_r\left(\Tilde{\pi}^\star, f^{\star}\right) - J_r\left(\Tilde{\pi}^\star, \mu_{n^\star}\right)| \leq \frac{\varepsilon}{2}$ and $|J_r\left(\hat \pi, \mu_{n^\star}\right) - J_r\left(\hat \pi, f^{\star}\right)| \leq \frac{\varepsilon}{2}$, which follow from Equation \ref{bound_reward:final} combined with $\Tilde{\pi}^\star,\hat \pi \in \Pi^{n^\star}_{\text{safe}}$.
\end{proof}
The above proof, together with that of \Cref{thm:complexity} in Appendix~\ref{appendix:sample_complexity} concludes the proof of \Cref{thm:merged}. In the remainder of this section, we provide the lemmas used in the optimality proof.

\begin{lemma}[adapted from Lemma A.3 of \citet{as2025actsafe}] \label{lemma:concentration_sigma}
    Let Assumptions \ref{assump:process_noise}-\ref{assumption:continuity_sigma} hold, and consider
    \begin{align*}
        x_{t+1}=\mu_n(x_t,\pi(x_t)) + w_t\\
        x_{t+1}'=\mu_n(x_t',\pi'(x_t')) + w_t
    \end{align*}
    for two policies $\pi,\pi'$ s.t $\|\pi'-\pi\|_{\infty}\leq \varepsilon_\pi$, assuming the same initial state distribution.
    Then,
    \begin{equation*}
        |J_{s_n}(\pi',\mu_n) - J_{s_n}(\pi, \mu_n)| \leq \sqrt{\varepsilon_\pi} C_\pi
    \end{equation*}
    where $C_\pi^n=T\left(L_s + TL_{\mu_n}\frac{\sqrt{d_x\sigma_{\max}\varepsilon_\pi}}{\sigma_w}\right)$.
    
\end{lemma}
\begin{proof}
    From Lemma 5 of \citet{OptiActive} ({Difference in Policy Performance}) applied on the $\mu_n$ dynamics (instead of $f^{\star}$ as in the original notation) we have
    \begin{equation*}
        J_{s_n}(\pi',\mu_n) - J_{s_n}(\pi, \mu_n) =\mathbb{E}_{\tau_{\pi'}^{\mu_n}}\left[\sum_{t=0}^{T-1}A_{s_n,t}(\pi,x_t',\pi'(x_t'))\right]
    \end{equation*}
    where $A_{s_n,t}(\pi,x_t',\pi'(x_t'))=\mathbb{E}_{\tau^{\pi}}\left[s_n(x_t',\pi'(x_t')) + J_{s_n,t+1}(\pi, \mu_n, x_{t+1}') - J_{s_n,t}(\pi,\mu_n,x_t')\right]$. Notice that we can rewrite $J_{s_n,t}(\pi,\mu_n,x_t')=s_n(x_t',\pi(x_t')) + \mathbb{E}_{\tau^{\mu_n}_{\pi}}\left[J_{s_n,t+1}(\pi,\mu_n,\Tilde{x}_{t+1})\right]$ for $\Tilde{x}_{t+1}=\mu_n(x_t',\pi(x_t')) + w_t$. We can thus substitute this in the above to get:
    \begin{align*}
        &J_{s_n}(\pi',\mu_n) - J_{s_n}(\pi, \mu_n) =\mathbb{E}_{\tau_{\pi'}^{\mu_n}}\left[\sum_{t=0}^{T-1}A_{s_n,t}(\pi,x_t',\pi'(x_t'))\right]\\
        &=\mathbb{E}_{\tau_{\pi'}^{\mu_n}}\left[\sum_{t=0}^{T-1}\left( s_n(x_t',\pi'(x_t')) - s_n(x_t',\pi(x_t')) \right) \right] \\
        &\quad + \mathbb{E}_{\tau_{\pi'}^{\mu_n}}\left[\sum_{t=0}^{T-1} \mathbb{E}_{x_{t+1}'\mid x_t',\pi'(x_t')}\left[J_{s_n,t+1}(\pi,\mu_n,x_{t+1}') \right] - \mathbb{E}_{\Tilde{x}_{t+1}\mid x_t',\pi(x_t')}\left[J_{s_n,t+1}(\pi,\mu_n,\Tilde{x}_{t+1}) \right] \right] \\   
    &\overset{(i)}{\leq}
\mathbb{E}_{\tau_{\pi'}^{\mu_n}}\Bigg[
\sum_{t=0}^{T-1}
\min\Big\{
L_s\|\pi'(x_t') - \pi(x_t')\|^{1/2}, \,
2\sqrt{d_x\sigma_{\max}}
\Big\}
\Bigg] \\
&\quad +
\mathbb{E}_{\tau_{\pi'}^{\mu_n}}\Bigg[
\sum_{t=0}^{T-1}
\Bigg\{
\sqrt{
\mathbb{E}_{x_{t+1}'\mid x_t',\pi'(x_t')}
\big[
J^2_{s_n,t+1}(\pi,\mu_n,x_{t+1}')
\big]
}
\\
&\qquad\qquad \times
\min\Big\{
\frac{\|\mu_n(x_t',\pi'(x_t')) - \mu_n(x_t',\pi(x_t'))\|}{\sigma_w},
1
\Big\}
\Bigg\}
\Bigg] \qquad \text{(\citet{KakadeInfo}, Lemma C.2)} \\
&\overset{(ii)}{\leq}
\mathbb{E}_{\tau_{\pi'}^{\mu_n}}\Bigg[
\sum_{t=0}^{T-1}
\Bigg\{
\min\Big\{
L_s\|\pi'(x_t') - \pi(x_t')\|^{1/2}, \,
2\sqrt{d_x\sigma_{\max}}
\Big\} \\
&\qquad\qquad +
T\sqrt{d_x\sigma_{\max}}
\min\Big\{
\frac{L_{\mu_n}\|\pi'(x_t') - \pi(x_t')\|}{\sigma_w},
1
\Big\}
\Bigg\}
\Bigg] \\
&\leq
\mathbb{E}_{\tau_{\pi'}^{\mu_n}}\Bigg[
\sum_{t=0}^{T-1}
\Bigg\{
L_s\|\pi'(x_t') - \pi(x_t')\|^{1/2}
+
T\sqrt{d_x\sigma_{\max}}
\frac{L_{\mu_n}\|\pi'(x_t') - \pi(x_t')\|}{\sigma_w}
\Bigg\}
\Bigg]
\end{align*}
    where in $(i)$-$(ii)$ we used the Hölder bounds provided by Lemma \ref{lemma:continuity_sigma}. Finally, using that $\|\pi'-\pi\|_{\infty}\leq \varepsilon_\pi$ by assumption we get
    \begin{align*}
J_{s_n}(\pi',\mu_n) - J_{s_n}(\pi, \mu_n)
&\leq \mathbb{E}_{\tau_{\pi'}^{\mu_n}}\Bigg[
\sum_{t=0}^{T-1} \Big\{
L_s \|\pi'(x_t') - \pi(x_t')\|^{1/2} \\
&\quad + T\sqrt{d_x\sigma_{\max}} 
\frac{L_{\mu_n} \|\pi'(x_t') - \pi(x_t')\|}{\sigma_w}
\Big\}
\Bigg]\\
        &\leq\sqrt{\varepsilon_\pi} T\left(L_s + TL_{\mu_n}\frac{\sqrt{d_x\sigma_{\max}\varepsilon_\pi}}{\sigma_w}\right)\eqqcolon \sqrt{\varepsilon_\pi} C_\pi^n.
    \end{align*}
    
\end{proof}

\begin{lemma}\label{lemma:pi_eps_subset_pi_safe}
Let Assumptions \ref{assump:process_noise}-\ref{assump:information_gain} hold. Fix $n^\star$ as in \eqref{eq: sample complexity} and $\varepsilon$, $\zeta$ and $\xi$ as in Theorem \ref{thm:optimality}. 
    Then, at episode $n^\star$, it holds $ \Pi_{\xi}^{\star, c}\subseteq\Pi^{n^\star}_{\text{safe}}$.
\end{lemma}
\begin{proof}
    Assume, for the sake of contradiction, that $ \Pi_{\xi}^{\star, c}\setminus\Pi^{n^\star}_{\text{safe}}$ is nonempty.
    First, we can choose $\varepsilon_\pi>0$ sufficiently small to satisfy $2\sqrt{\varepsilon_\pi} C_\pi^{n^\star} \lambda_{n^\star}^c \leq\xi-\varepsilon - \Delta_\zeta$, where $C_\pi^{n^\star}$ is defined in Lemma \ref{lemma:concentration_sigma} and constant (since we consider a fixed $n^\star$.
    Then, under the  hypothesis that $\Pi_{\xi}^{\star, c}\setminus \Pi_{\textrm{safe}}^{n^\star}$ is nonempty and under the definition of path-connectedness of $\Pi_{\xi}^{\star, c}$, there exists $\pi_{\text{safe}} \in\Pi^{n^\star}_{\text{safe}}$
    and $\pi'\in \Pi_{\xi}^{\star, c}\setminus\Pi^{n^\star}_{\text{safe}}$ such that $\|\pi' - \pi_{\text{safe}}\|_{\infty} \leq \varepsilon_\pi$. This is because 
    by the definition of path-connectedness of $\Pi_{\xi}^{\star, c}$, we can connect a policy in $ \Pi_{\xi}^{\star, c}\setminus\Pi^{n^\star}_{\text{safe}}$ to some policy in $\Pi_\text{prior} \subseteq \Pi^{n^\star}_\text{safe}$ with a continuous line $\rho$ in the policy space, and then choose along this line $\pi_{\text{safe}}$ and $\pi'$ arbitrarily close and such that they lie right inside and right outside of $\Pi^{n^\star}_{\text{safe}}$ respectively. 

    For these policies we can then apply Lemma \ref{lemma:concentration_sigma} to bound:
    \begin{align*}
        J_{s_{n^\star}}(\pi',\mu_{n^\star})\leq J_{s_{n^\star}}(\pi_{\text{safe}},\mu_{n^\star}) + \sqrt{\varepsilon_\pi} C_\pi.
    \end{align*}

This bound can be substituted in when evaluating (\ref{bound_cost:every_policy})  along the trajectories generated by $\pi'$(as remember that (\ref{bound_cost:every_policy}) holds $\forall \pi \in \Pi$):
\begin{align}
    &|J_c(\pi',f^{\star}) - J_c(\pi', f_{n^\star}^m)| \leq 2\lambda_{n^\star}^cJ_{s_{n^\star}}(\pi',\mu_{n^\star})\\
    &\leq 2 \lambda_{n^\star}^c\left( J_{s_{n^\star}}(\pi_{\text{safe}},\mu_{n^\star}) + \sqrt{\varepsilon_\pi} C_\pi^{n^\star}\right)\\
    &\leq2\lambda_{n^\star}^c\left(d_\sigma^{n^\star} + \sqrt{\varepsilon_\pi} C_\pi^{n^\star}\right)\\
    &\leq \varepsilon + 2\sqrt{\varepsilon_\pi} C_\pi^{n^\star} \lambda_{n^\star}^c\leq  \xi-\Delta_\zeta.
\end{align}
From this it follows, for all $m\in[M]$:
\begin{equation}
    J_c\left(\pi',  f_{n^\star}^m\right)\leq J_c\left(\pi', f^{\star}\right) + \xi-\Delta_\zeta \leq d - \xi +  \xi-\Delta_\zeta \leq d - \Delta_\zeta.
\end{equation}
This means that $\pi'\in\Pi^{n^\star}_{\text{safe}}$, which is a contradiction, thus proving the statement.
\end{proof}

\begin{lemma}\label{lemma:continuity_sigma}
 Under assumption \ref{assumption:continuity_sigma}, we have for all $z,z'\in\calZ$ and $n\geq0$ that $|s_n(z)-s_n(z')|\leq L_s\|z-z'\|^{1/2}$ and $|\mu_n(z)-\mu_n(z')|\leq L_{\mu_n}\|z-z'\|$, where $L_s\coloneqq\sqrt{2d_xL_k}$ and $L_{\mu_n}\coloneqq \sqrt{\sum_{j=1}^{d_x}\left(L_\mu+\sqrt{nT}L_k\frac{\|y_{n,j}-\mu_j\|}{\sigma_w^2}\right)^2}$.
\end{lemma}
\begin{proof}
    For $s_n$ we can  write:
    \begin{align*}
        |s_n(z)-s_n(z')|&=\mid\|\sigma_n(z)\|-\|\sigma_n(z')\|\mid\\
        &\overset{(i)}{\leq}\|\sigma_n(z)-\sigma_n(z')\|\\
        &=\sqrt{\sum_{j=1}^{d_x}(\sigma_{n,j}(z)-\sigma_{n,j}(z'))^2} \\
        &\overset{(ii)}{\leq}\sqrt{d_x}\sqrt{k(z,z)-k(z,z')+k(z',z')-k(z',z)}\\
        &\overset{(iii)}{\leq}\sqrt{2d_xL_k}\|z-z'\|^{1/2}
    \end{align*}
where $(i)$ follows from the reverse triangular inequality, $(ii)$ follows by applying Lemma 12 of~\citet{HURCL} componentwise to the scalars $\sigma_{n,j}$ (note that this does not rely on having a bounded domain), and $(iii)$ follows from \Cref{assumption:continuity_sigma}.

For the mean, we can write for each $j\in[d_x]$:
\begin{align*}
    |\mu_{n,j}(z)-\mu_{n,j}(z')|&\leq|\mu_j(z)-\mu_j(z')|+\|k_n(z)-k_n(z')\|\|\left(K_n+\sigma_w^2I\right)^{-1}(y_{n,j}-\mu_j)\|\\
    &\leq L_\mu\|z-z'\|+\frac{1}{\sigma_w^2}\|y_{n,j}-\mu_j\|\|k_n(z)-k_n(z')\|\\
    &\leq L_\mu\|z-z'\|+\frac{1}{\sigma_w^2}\|y_{n,j}-\mu_j\|\sqrt{nT}\max_{i\in[nT]}\mid k(z_i,z)-k(z_i,z')\mid\\
    &\leq L_\mu\|z-z'\|+\frac{1}{\sigma_w^2}\|y_{n,j}-\mu_j\|\sqrt{nT}L_k\|z-z'\|\\
    &\leq \left(L_\mu+\sqrt{nT}L_k\frac{\|y_{n,j}-\mu_j\|}{\sigma_w^2}\right)\|z-z'\|
\end{align*}
Therefore, $\|\mu_n(z)-\mu_n(z')\|\leq L_{\mu_n}\|z-z'\|$.
\end{proof}

\section{Experiment Details} \label{appendix:experiments}
In this section, we provide additional details about the experimental evaluation presented in Section~\ref{section:experiments}.
\subsection{GP Experiments}
The code for reproducing our experiments is available at \href{https://github.com/lucavignola/safe-model-based-exploration}{this repository}, which includes further details on the implementation, experimental setup, and hyperparameters. All experiments were conducted using a single NVIDIA RTX 4090 GPU per run, along with 5 CPU cores and 40GB of RAM. Under this setup, each experiment completes in under 1.5 hours.

\paragraph{Experimental setup.}
We evaluate \algname{SBSRL} on the \textsc{Pendulum} and \textsc{Cartpole} environments, focusing on the \textsc{SwingUp} task. In \textsc{Cartpole}, following \citet{as2025actsafe}, we assume access to an initial offline dataset to ensure safe initialization (as required by Assumptions~\ref{assump:rkhs}-\ref{assump:prior_knowledge}), whereas no initial offline dataset is required for \textsc{Pendulum}. For completeness, \Cref{fig:pendulum_multi} complements \Cref{fig:experiments_gp}, including both cumulative rewards and training curves for the same safety experiment.
Among the considered baselines, \algname{ActSafe} includes an initial purely exploratory phase, during which rewards are not optimized, before transitioning to exploitation until the end of the simulation. In \Cref{fig:gp_pendulum_exploration} we report two possible choices of the transition episode.
\begin{figure} 
    \centering
    
    \begin{minipage}{0.48\textwidth}
        \centering
    \includegraphics{figures/barplot_reward_cost_legend_PendulumGP_multi_notes.pdf}
    \end{minipage}
    \hfill
    \begin{minipage}{0.48\textwidth}
        \centering
        \includegraphics{figures/barplot_reward_cost_legend_CartPoleGP_multi_notes.pdf}
    \end{minipage}
    \includegraphics{figures/legend_only_bar_plot_CartPoleGP_multi_notes.pdf}

    \begin{minipage}{0.48\textwidth}
        \centering
    \includegraphics{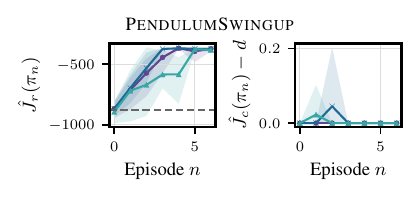}
    \end{minipage}
    \hfill
    \begin{minipage}{0.48\textwidth}
        \centering
        \includegraphics{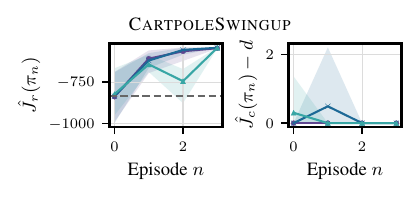}
    \end{minipage}
\includegraphics{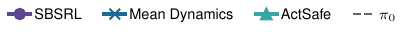}
 \caption{GP experiments demonstrating the effectiveness of using sampling to incentivize safety. The curves report mean and $95$ percentile interval over five seeds, while the bar charts visualize  mean and standard deviation of the normalized cumulative returns and maximum cost violations over five seeds.}
\label{fig:pendulum_multi}
\end{figure}

\paragraph{Rewards and constraints.}
We adopt reward and constraint definitions consistent with prior work on safe model-based reinforcement learning \citep{as2025actsafe}. For both \textsc{Pendulum} and \textsc{Cartpole} environments, let $\theta$ denote the pole angle, $\omega$ its angular velocity, and $u$ the control input. We define the angular deviation from a desired target angle $\theta_{\text{target}}$ as $\Delta \theta = \theta - \theta_{\text{target}}$.
In the \textsc{Pendulum} environment, the objective penalizes deviations from the target configuration as well as large angular velocities and control inputs. The reward and cost are given by
\begin{equation*}
    r_{\text{Pendulum}} = -\left(\Delta \theta^2 + 0.1 \, \omega^2 + 0.02 \, u^2 \right), \quad
    c_{\text{Pendulum}} = |\omega|, \quad d = 6.
\end{equation*}
For the \textsc{Cartpole} system, we additionally consider the cart position $p$ and velocity $v$. The reward encourages stabilization of both the pole and the cart while penalizing control effort:
\begin{equation*}
    r_{\text{Cartpole}} = -\left(\Delta \theta^2 + p^2 + 0.1 \left( v^2 + \omega^2 \right) \right) - 0.01 \, u^2,
\end{equation*}
while the cost captures violations of a position constraint on the cart,
\begin{equation*}
    c_{\text{Cartpole}} = |p|, \quad d = 1.5.
\end{equation*}

\paragraph{Policy optimization.}
In all GP experiments, we use iCEM~\citep{Pinneri2020SampleefficientCM}
to approximately solve Problem~\ref{pb:1} through the fixed-penalty
objective
\begin{equation*}
\arg\max_{\pi \in \Pi}
J_r(\pi,\mu_n)
-
\lambda_c \sum_{m=1}^{M}
\max\!\left(
J_c(\pi,f_n^m)-d+\Delta_\zeta,\,0
\right)
-
\lambda_\sigma
\max\!\left(
d_\sigma^n-J_{s_n}(\pi,\mu_n),\,0
\right),
\end{equation*}
where $\lambda_c$ and $\lambda_\sigma$ are fixed\footnote{Constraint satisfaction could
be enforced more tightly by updating the multipliers within an
additional primal--dual optimization loop, at the cost of increased
computation.} penalty coefficients used to discourage constraints violations \citep{as2025actsafe}. For each candidate action sequence $\{a_t\}_{t=0}^{T-1}$ we obtain $M$ trajectories by rolling out the actions on each of the $M$ sampled dynamics; these trajectories are then used to evaluate $J_c(\pi,  f_n^m)$ and enforce the safety constraint. Finally, we fix the number of samples to $M=30$ across all experiments. As explained in Section~\ref{section:algorithm}, the number of samples controls a trade-off between computational cost and conservatism: as using more samples than prescribed by Lemma~\ref{lem: sampling} does not compromise safety, one can in principle select the largest number of samples permitted by the available computational budget, and subsequently choose the constraint tightening accordingly. For instance, in our experiments we match the number of samples used by \citet{as2025actsafe} to ensure comparable computational cost, and we then empirically observe that smaller constraint tightenings than those required by \algname{ActSafe} are sufficient to maintain safety.

\paragraph{Dynamics sampling.}
We consider three methods for generating the sampled dynamics
trajectories used by the planner. In the \emph{prior-path} variant, following the construction analyzed in Section~\ref{section:algorithm}, we
draw $M$ approximate global samples from the GP prior using the
decoupled Matheron construction~\citep{EfficientSampling}. These paths
are retained throughout learning and recursively truncated by the GP
confidence sets constructed after each model update. The
\emph{posterior-path} variant uses the same pathwise construction, but
draws new functions from the updated GP posterior at every episode
before applying the recursive truncation; its
connection to the theoretical analysis is discussed in
Appendix~\ref{appendix:discussions}. Finally, the \emph{TS1} variant
\citep{chua2018deep} samples each transition independently from the
corresponding posterior marginal of the GP; note that this differs from the pathwise conditioning procedure (as it injects independent noise at each timestep) and is thus not covered by our theory, but provides a useful practical approximation that we can easily compare with baselines \citep{as2025actsafe}.

\paragraph{Sampling ablation.}
To isolate the effect of the sampling implementation, we compare the
three variants using the same iCEM planner, penalty formulation, and
number of samples $M=30$. As illustrated in 
\Cref{tab:gp_sampling_ablation}, all three implementations obtain comparable returns in both
environments. Furthermore, the prior-path and TS1 variants incur no observed
 violations, while the posterior-path
variant only incurs small violations in
\textsc{CartpoleSwingup}. TS1 is computationally efficient
and matches the sampling procedure used by \citet{as2025actsafe}, enabling a controlled comparison that isolates the
contribution of our exploration constraint. We therefore use TS1 in the main experiments of \Cref{section:experiments}.

\begin{table}[t]
    \centering
    \caption{Comparison of several GP sampling implementations. We report the unnormalized reward at the final episode. Instead, Avg. violation\ and Max violation\ are respectively the maximum over episodes of the average and maximum accumulated-cost violation across five evaluation trajectories. Entries report the mean and standard deviation across five seeds.}

    \label{tab:gp_sampling_ablation}
    \begin{tabular}{lccc}
        \toprule
        Method
        & Reward
        & Avg.\ violation
        & Max\ violation \\
        \midrule
        \multicolumn{4}{c}{\textsc{PendulumSwingup}} \\
        \midrule
        Prior paths + truncation
        & $-190.26 \pm 1.11$
        & $0$ & $0$  \\
        Posterior paths + truncation
        & $-195.62 \pm 7.03$
        & $0$ & $0$  \\
        TS1
        & $-185.46 \pm 2.56$
        & $0$ & $0$ \\
        \midrule
        \multicolumn{4}{c}{\textsc{CartpoleSwingup}} \\
        \midrule
        Prior paths + truncation
        & $-274.78 \pm 1.21$
        & $0$ & $0$ \\
        Posterior paths + truncation
        & $-275.07 \pm 1.05$
        & $0.0057 \pm 0.0127$
        & $0.0254 \pm 0.0568$
        \\
        TS1
        & $-274.80 \pm 1.21$
        & $0$ & $0$ \\
        \bottomrule
    \end{tabular}
\end{table}

\paragraph{Exploration constraint.}
While the theoretical value of $d_\sigma^n$ given in \Cref{thm:merged} is difficult to estimate in practice, we find that \algname{SBSRL} is not overly sensitive to it. Thus, we treat this value as an hyperparameter in our experiments. Choosing the right value may improve performance; however, safety is not affected, since it is, at least in theory (ignoring approximation errors introduced by solving the constrained optimization problem), decoupled from exploration.
As an additional simplification, we keep the exploration threshold $d_\sigma^n$ fixed across episodes rather than progressively tightening it; although both variants are implemented in our code, we found empirically that this simplification preserves performance while improving training stability. Finally, note that, as prescribed by \Cref{alg:mbrl1}, the exploration constraint is automatically deactivated once it becomes infeasible.
Empirically, \Cref{fig:ablation_exploration} shows how setting $d_\sigma^0$ too large can deteriorate performance, since the constraint is deactivated early and the behavior quickly converges to that of $d_\sigma^0 = 0$. Conversely, choosing it very small can in principle lead to better asymptotic performance as suggested by \Cref{thm:merged}, but it also increases sample complexity and thus can ultimately reduce the cumulative reward, as illustrated in the figure.

\begin{figure}
\vspace{-0.4cm}
    \centering
    \begin{minipage}{0.49\textwidth}
        \centering
\includegraphics{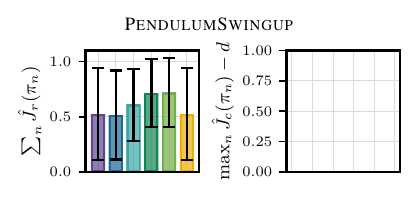}
    \vspace{-0.5cm}
    \end{minipage}
    \hfill
    \begin{minipage}{0.49\textwidth}
        \centering
    \includegraphics{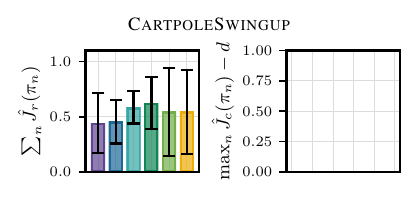}
    \vspace{-0.5cm}
    \end{minipage}
\includegraphics[width=\textwidth]{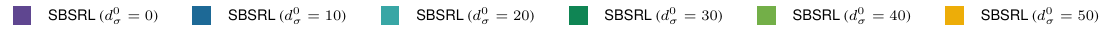}
    \caption{GP experiments ablating different values of $d_\sigma^0$ in the exploration constraint. The bar charts report the mean and standard deviation over five seeds for the normalized cumulative return and maximum cost violation.}
    \vspace{-0.5cm}
    \label{fig:ablation_exploration}
\end{figure}

\subsection{Hardware Experiment}
While Gaussian Processes closely align with our theoretical framework, their cubic complexity in the number of data points limits scalability to high-dimensional control tasks. To address this, we also consider a scalable instantiation of \algname{SBSRL} based on neural network dynamics models, enabling application to hardware, but also to standard continuous control benchmarks~\citep{Ray2019,challengesrealworld}. This "deep" instantiation of \algname{SBSRL} is only provided as an empirical heuristic and does not inherit the GP theorems. The code for this implementation is available at \href{https://github.com/lucavignola/safe-learning}{this repository}.

\paragraph{Implementation details.}
We instantiate \algname{SBSRL} within a model-based actor-critic framework inspired by MBPO~\citep{mbpo}. The system dynamics are modeled using a probabilistic ensemble of neural networks~\citep{lakshminarayanan2017simple,chua2018deep}, which enables multiple predictions of the dynamics for enforcing safety and provides an estimate of epistemic uncertainty~\citep{depeweg2018decomposition} via the empirical standard deviation across ensemble members, which we exploit to implement the exploration constraint.
To approximately solve Problem~(\ref{pb:1}), we employ a mixture of real and model generated rollouts, and perform a fixed number of interleaved actor-critic updates, following standard practice in model-based RL \citep{mbpo}. In contrast to standard MBPO-style implementations, in order to approximate the sampling-based nature of our method we retain all ensemble predictions during rollouts, effectively simulating multiple dynamics realizations and enforcing safety constraints across them.
To support this, we augment the cost critic to condition on the index $m$ of the dynamics sample, thereby approximating
\begin{equation*}
    J_c(\pi,  f_n^m)\approx \mathbb{E}_{\pi, x\sim\rho_0} \left[ Q_c(x, \pi(x), m) \right] .
\end{equation*}
Exploiting the structural similarity between $J_c$ and $J_{s_n}$, we adopt the same approach to estimate the cumulative uncertainty along trajectories
\begin{equation*}
    J_{s_n}(\pi,  f_n^m)\approx\mathbb{E}_{\pi, x\sim\rho_0} \left[ Q_\sigma(x, \pi(x), m) \right] .
\end{equation*}
This enables seamless integration of the exploration constraint within a standard actor-critic pipeline.

\paragraph{Design choices.}
Constraint satisfaction is enforced using an Augmented Lagrangian method~\citep{NoceWrig06}, which provides strong empirical performance, consistent with prior work~\citep{Li2021AugmentedLM,as2022constrained}, although other solvers could be used in principle.
Furthermore, while the theoretical analysis prescribes a sample size $M$ satisfying \Cref{lem: sampling} for GPs, we found that fixing $M = 5$ across all experiments (both on hardware and in simulation) was sufficient for the implementation using deep neural networks to capture uncertainty while maintaining computational efficiency. Finally, note that in this implementation both reward and cost functions are unknown and therefore estimated via learned critics, as is standard in actor-critic methods.
\paragraph{Hardware setup.}
We follow the experiment setup of \citet{wendl2026safe} and evaluate \algname{SBSRL} on a highly-dynamic remote-controlled race car operating at $60$ Hz. States are observed through a motion capture systems that tracks the vehicle trajectory and sends real-time measurements to the controller. While the underlying system is characterized by a $7$-dimensional state and a $2$-dimensional action space, we concatenate past observations and control inputs using a sliding window. This mitigates delays arising from both actuation and motion capture, resulting in a 45-dimensional state-action representation. Such a representation is challenging for GP-based methods and motivates the use of deep ensembles.
The task is to reach a target position while satisfying safety constraints, represented by obstacles that must be avoided by the vehicle (see Figure~\ref{fig:hardware_flagship}). Constraint violations are quantified using measurements from the motion capture system, and represent the kinetic energy dissipated during plastic collisions with the obstacles. The budget is set to $d = 15$, allowing for minor contacts with the tires while penalizing large collisions.
In order to initialize our algorithm with a prior model that approximately satisfies \Cref{assump:rkhs,assump:prior_knowledge}, we first collect 25K real-world transitions using a safe but suboptimal policy. This dataset is then used to empirically \emph{calibrate} the model (and critics) of \algname{SBSRL} offline. Finally, we deploy \algname{SBSRL} online on the real system. At each episode, a policy is rolled-out on the real system to collect more data, which is added to the replay buffer to update the model before the next episode. We repeat the experiment for $75$ episodes with three random seeds and report the mean and standard error of the cumulative rewards and costs after each iteration $n$ in Figure \ref{fig:hardware_experiment}. 
We compare \algname{SBSRL} against a version of the same algorithm that does not use pessimism to enforce safety, i.e. it only enforces the constraint on the mean critic rather than across the ensemble. To ensure a fair comparison, both methods are initialized with the same offline prior, which includes the dynamics model and critics that were trained on the offline dataset.

\section{Additional Experiments}\label{appendix:additional}
\subsection{Experiments in Simulated Environments}
We further evaluate \algname{SBSRL} on standard continuous control benchmarks, including SafetyGym~\citep{Ray2019} and RWRL~\citep{challengesrealworld}. In these experiments, we follow the setup of~\citet{as2025spidrsimpleapproachzeroshot,wendl2026safe}; we refer the reader to these works for a detailed description of the tasks and environments. For these experiments we use a single NVIDIA RTX 4090 GPU per run, with 5 CPU cores and 40 GB of RAM. Multiple random seeds are executed in parallel via a Slurm cluster. This setup enables us to train our algorithm in less than an hour and a half for each experiment.
\paragraph{Safe offline-to-online.}
Ensuring safety when training neural network dynamics models from scratch is challenging without prior knowledge, as the initial model is typically poorly calibrated~\citep{as2025actsafe}. Our theoretical framework assumes access to a sufficiently good prior (cf. \Cref{assump:rkhs,assump:prior_knowledge}), which is not directly satisfied in this setting.
To address this, we adopt an offline-to-online training pipeline. However, note that our theory is agnostic to how such a prior is obtained, so that alternative approaches such as sim-to-real could also be considered. Specifically, we assume access to an initial dataset collected by a suboptimal (and potentially unsafe) behavior policy. We first run our algorithm in an offline setting (i.e. using only the offline data) to learn a (empirically) calibrated dynamics model and associated critics. These are then used to initialize the online phase, allowing the agent to begin learning from a reliable model. The only modification required to run \algname{SBSRL} offline is to set $d_{\sigma}^n=0$, thereby deactivating the exploration constraint. This choice reflects the offline RL setting, where the agent should remain close to the data distribution rather than be encouraged to explore. If one instead wishes to impose explicit exploration penalties during the offline phase, as suggested by \citet{yu2020mopo}, our framework can naturally accommodate this by reversing the exploration constraint. In our experiments, however, this was not necessary, allowing us to retain a single practical algorithm for the full offline-to-online pipeline.
To fairly validate \algname{SBSRL} accounting for the variability in the offline initialization, we generate multiple priors by running $5$ independent offline training seeds on the same dataset. For each resulting prior, we then run $5$ independent online training seeds. This results in a total of $25$ runs per experiment, all of which are reported in our evaluation.

\paragraph{Benchmarks and baselines.}
We evaluate our method on several tasks from the RWRL benchmark~\citep{challengesrealworld} and SafetyGym~\citep{Ray2019}. 
In Figure \ref{fig:experiments_with_baselines}, we compare \algname{SBSRL} against two baselines from the literature. First, we consider \algname{MBPO}~\citep{mbpo}, which also serves as the base architecture for our implementation, as explained in the previous section. \algname{MBPO} is a model-based actor-critic method that uses a neural network ensemble to model the dynamics and generates rollouts via the TS1 scheme of~\citet{chua2018deep}, randomly selecting a different ensemble member at each step. Notably, \algname{MBPO} does not incorporate additional pessimism to enforce safety, which results in cost violations in some tasks. Second, we compare against \algname{SOOPER}~\citep{wendl2026safe}, a recent state-of-the-art safe RL algorithm. Its implementation is also based on an MBPO-like architecture, enabling a direct and fair comparison. In particular, we initialize both \algname{MBPO} and \algname{SOOPER} with a prior policy and backup critics that are constructed from the same offline run used to initialize \algname{SBSRL}.
Across all experiments, \algname{SBSRL} maintains safety at all episodes throughout training while improving rewards over the safe baselines. The plots report the mean and $95\%$ empirical percentile intervals over $25$ runs, generated from $5$ different offline priors as described above.
\begin{figure}
    \centering
    \includegraphics[width=0.8\linewidth]{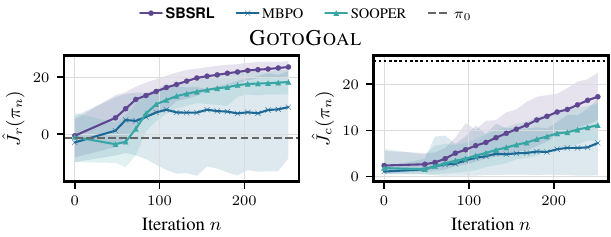}\\
\includegraphics[width=0.8\linewidth]{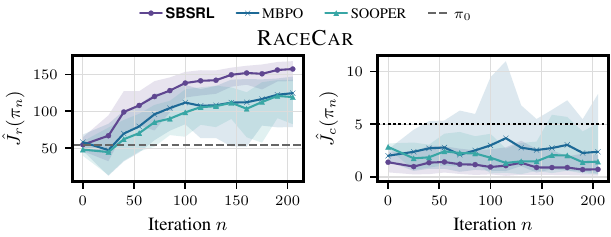}\\
\includegraphics[width=0.8\linewidth]{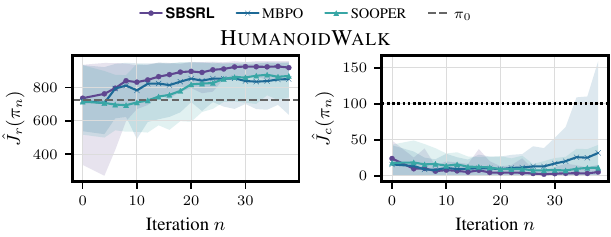}\\
\includegraphics[width=0.8\linewidth]{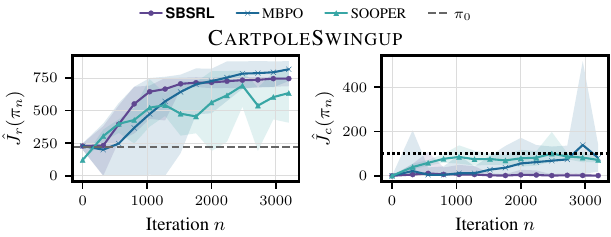}\\
\includegraphics[width=0.8\linewidth]{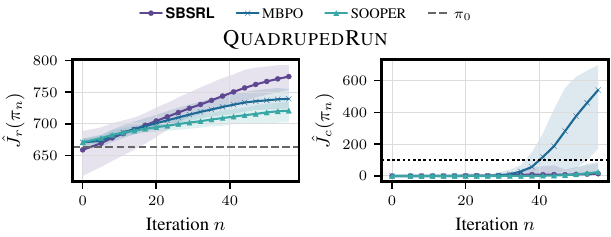}\\
    \caption{Learning curves comparing \algname{SBSRL} against standard baselines. The performance is evaluated in the environments \textsc{GoToGoal}, \textsc{RaceCar}, \textsc{Humanoid}, \textsc{Cartpole} and  \textsc{Quadruped} . We report the mean and $95$ percentile interval.
    }
    \label{fig:experiments_with_baselines}
\end{figure}

\paragraph{Constraint satisfaction through sampling.}
\looseness=-1In Figure \ref{fig:experiments_vertical} we provide a comparison of \algname{SBSRL} against a mean baseline that only uses the mean of the multiple critic predictions across ensemble members for planning, rather than accounting for each sampled dynamics model. The results show that \algname{SBSRL} achieves comparable task performance while consistently satisfying the safety constraints across all runs. We report the mean and $95$ percentile intervals over all $25$ runs, highlighting robustness of \algname{SBSRL} to the offline prior. This explains the apparent high variance at initialization and provides a more informative measure of robustness than reporting standard errors alone.
\begin{figure}
    \centering
    \includegraphics[width=0.8\linewidth]{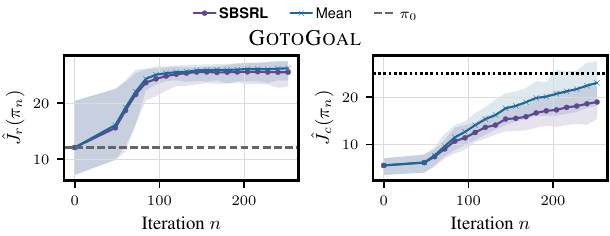}
    \includegraphics[width=0.8\linewidth]{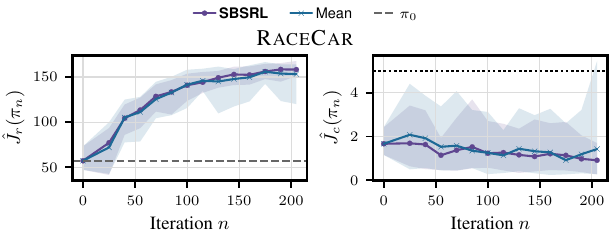}
    \includegraphics[width=0.8\linewidth]{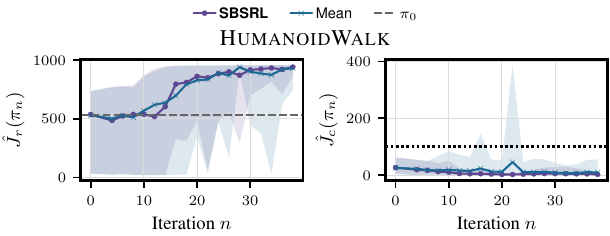}
    \includegraphics[width=0.8\linewidth]{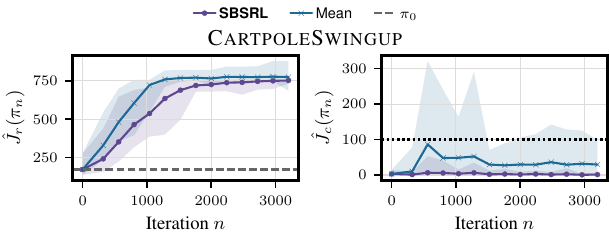}
    \includegraphics[width=0.8\linewidth]{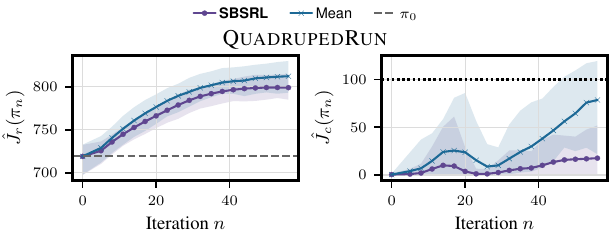}
    \caption{
    Learning curves comparing \algname{SBSRL} against the mean baseline. The performance is evaluated in the environments \textsc{GoToGoal}, \textsc{RaceCar}, \textsc{Humanoid}, \textsc{Cartpole} and  \textsc{Quadruped}. We report the mean and $95$ percentile interval.
    }
    \label{fig:experiments_vertical}
\end{figure}

\paragraph{Ensemble size.}
Although for neural-network ensembles the variance of the ensemble predictions can be used as a proxy for epistemic uncertainty, \Cref{thm:safety} also relies on the existence of a $\zeta$-close sample, which is not guaranteed for such a deep ensemble. Nevertheless, this practical implementation allows us to scale our algorithmic principles to high-dimensional control tasks. In \Cref{fig:ensemble_size}, we additionally provide a sweep over the ensemble size $M$ in the \textsc{Cartpole} environment, showing the empirical sensitivity of the deep implementation to the ensemble size. We note that $M=5$ is already sufficient to maintain safety throughout learning in this experiment, while increasing the size further improves the margin to the safety boundary with a slight decrease in performance.
\begin{figure}
    \centering
    \includegraphics[width=0.8\linewidth]{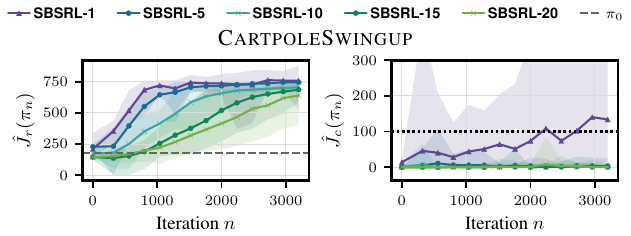}
    \caption{Learning curves comparing the deep instantiation of \algname{SBSRL} across different ensemble sizes $M$ (denoted by the label SBSRL-$M$). The performance is evaluated in the \textsc{Cartpole} environment.  We report the mean and standard deviation over 25 seeds.
    }
    \label{fig:ensemble_size}
\end{figure}

%%%%%%%%%%%%%%%%%%%%%%%%%%%%%%%%%%%%%%%%%%%%%%%%%%%%%%%%%%%%

\newpage

\end{document}